\documentclass{article} 
\PassOptionsToPackage{table}{xcolor}
\usepackage{iclr2027_conference,times}

\usepackage{amsmath,amsfonts,bm}

\def\eqref#1{equation~\ref{#1}}

\def\1{\bm{1}}

\DeclareMathAlphabet{\mathsfit}{\encodingdefault}{\sfdefault}{m}{sl}
\SetMathAlphabet{\mathsfit}{bold}{\encodingdefault}{\sfdefault}{bx}{n}

\newcommand{\E}{\mathbb{E}}

\newcommand{\R}{\mathbb{R}}

\usepackage{array}

\usepackage{hyperref}
\usepackage{cleveref}
\usepackage{url}
\usepackage{enumitem}

\usepackage{titlesec}

\usepackage{booktabs}
\usepackage{comment}
\usepackage{siunitx}
\usepackage{graphicx}
\usepackage{subcaption}
\usepackage{amsthm}
\usepackage{thm-restate}

\definecolor{cdcorange}{HTML}{E8710A}
\definecolor{lerpteal}{HTML}{0E8A99}
\definecolor{epclosed}{HTML}{2E7D32}
\definecolor{epopen}{HTML}{6A1B9A}
\definecolor{mdheld}{HTML}{2B2F36}
\definecolor{mdref}{HTML}{999999}
\definecolor{heroorange}{HTML}{E8710A}     
\definecolor{corescaf}{HTML}{6F7681} 
\definecolor{nmpblue}{HTML}{529CE8}   
\definecolor{lidgreen}{HTML}{21A863}  

\definecolor{transitioncolor}{HTML}{E8710A}
\definecolor{gaussiancolor}{HTML}{0E8A99}
\definecolor{knncolor}{HTML}{4D4D4D}
\definecolor{knnrefcolor}{HTML}{9A9A9A}
\colorlet{winrow}{heroorange!12!white}

\definecolor{knn}{HTML}{4D4D4D}

\definecolor{heroorange}{HTML}{E87524}

\colorlet{winrow}{heroorange!12!white}

\usepackage{caption}
\usepackage{wrapfig}

\usepackage{amsmath}

\usepackage{amssymb}

\usepackage{tikz-cd}
\usepackage{tabularx}
\usetikzlibrary{fit, backgrounds}   

\newtheorem{globalcounter}{Theorem}[section]

\newtheorem{proposition}[globalcounter]{Proposition}
\newtheorem{lemma}[globalcounter]{Lemma}
\newtheorem{definition}[globalcounter]{Definition}

\titlespacing*{\paragraph}{0pt}{0.1em}{0.4em}

\usepackage[most]{tcolorbox}

\colorlet{principlebg}{heroorange!4!white}
\colorlet{principleac}{heroorange}

\colorlet{takeawaybg}{gaussiancolor!4!white}
\colorlet{takeawayac}{gaussiancolor}

\newtcolorbox{principle}[1][]{%
  enhanced,
  breakable,
  sharp corners,
  frame hidden,
  colback=principlebg,
  boxrule=0pt,
  borderline west={0.8pt}{0pt}{principleac},
  left=6pt,
  right=5pt,
  top=3pt,
  bottom=3pt,
  before skip=6pt,
  after skip=6pt,
  boxsep=0pt,
  #1
}

\newtcolorbox{takeaway}[1][]{%
  enhanced,
  breakable,
  sharp corners,
  frame hidden,
  colback=takeawaybg,
  boxrule=0pt,
  borderline west={0.8pt}{0pt}{takeawayac},
  left=6pt,
  right=5pt,
  top=3pt,
  bottom=3pt,
  before skip=6pt,
  after skip=6pt,
  boxsep=0pt,
  #1
}

\newcounter{ManuCounter}

\newcounter{IoloCounter}

\newcounter{MichaelCounter}

\colorlet{reflink}{heroorange!80!black}

\hypersetup{
    colorlinks=true,
    linkcolor=reflink,
    citecolor=reflink,
    urlcolor=reflink
}

\title{Intrinsic Geometry Matching: learning\\ observation-equivariant process geometry}

\author{Manuel Lecha  \\
AITHYRA\\
\texttt{mlecha@aithyra.ac.at} \\
\And
Michael Bronstein \\
University of Oxford, \\
AITHYRA \\
\And
Iolo Jones \\
University of Oxford, \\
AITHYRA \\
\texttt{iolo.jones@cs.ox.ac.uk} \\
}

\iclrfinalcopy

\begin{document}

\maketitle
\fancyhead{}
\renewcommand{\headrulewidth}{0pt}

\begin{abstract}
Real data are often arbitrary, indirect observations of an unknown underlying process, and those observations can distort its geometry.
Existing geometric methods equip data with the induced geometry of the ambient space, but this entangles the true geometry of the process with the distortion of the observation.
To address this, we introduce \textbf{Intrinsic Geometry Matching (IGM)}: a scalable approach to learn the intrinsic geometry of the underlying process.
Specifically, we learn the \textit{carré du champ} operator of the process, allowing us to recover its diffusion geometry.
Our method uses the fact that data often come with natural perturbations, like steps in a time series, to recover the intrinsic carré du champ while transforming naturally with changes in observation: it is \textbf{observation-equivariant}.
Furthermore, since this equivariance relates the geometry of different spaces, we can train a model using additional sources of data, even when that data lies in another space.
We evaluate IGM on molecular dynamics, image, and video data and find better performance and scaling than existing neural and graph-based geometric tools.
\end{abstract}

\section{Introduction}
\label{sec:intro}

Real data are often arbitrary, indirect observations of an unknown underlying process.
In biology, cells, tissues, genomes, and proteins evolve and adapt, but can only be observed through noisy and often destructive experiments. 
Climate and weather systems have complex, chaotic structure, which is only partially captured by a sparse network of sensors and satellites.
These processes are governed by physical laws and constraints, so the data we collect have \textit{geometry}.
This geometry can tell us about the structure of the underlying system, but the geometry we observe is distorted by the nature of the observation.
For example, the distance between two satellite images as vectors is not the same as the physical `distance' between the weather systems they observe.
As the frontier of scientific AI moves from representing states to representing processes, the corresponding geometric challenge is to disentangle the true \textbf{process geometry} from the arbitrary \textbf{observation geometry}.

Scientific modelling is fundamentally concerned with systems under change: trajectories, perturbations, constraints, interventions, and controls. 
The mathematical language of change is \textbf{calculus}, which works with notions of derivative. 
On nonlinear manifolds, this is generalised by \textbf{Riemannian geometry}, whose basic primitive is the Riemannian metric. 
This picture provides a useful geometric intuition for machine learning as the \textit{Manifold Hypothesis}, which posits that high-dimensional data actually lie on a low-dimensional manifold \citep{tenenbaum2000global, roweis2000nonlinear, bengio2012representation}.
In reality, it is mostly just a metaphor: many complex scientific systems exhibit varying dimensionality, branches, singularities and discrete–continuous structure, so a smooth manifold of fixed dimension is too rigid an operational model.

\begin{figure}
    \centering
    \includegraphics[width=1\linewidth]{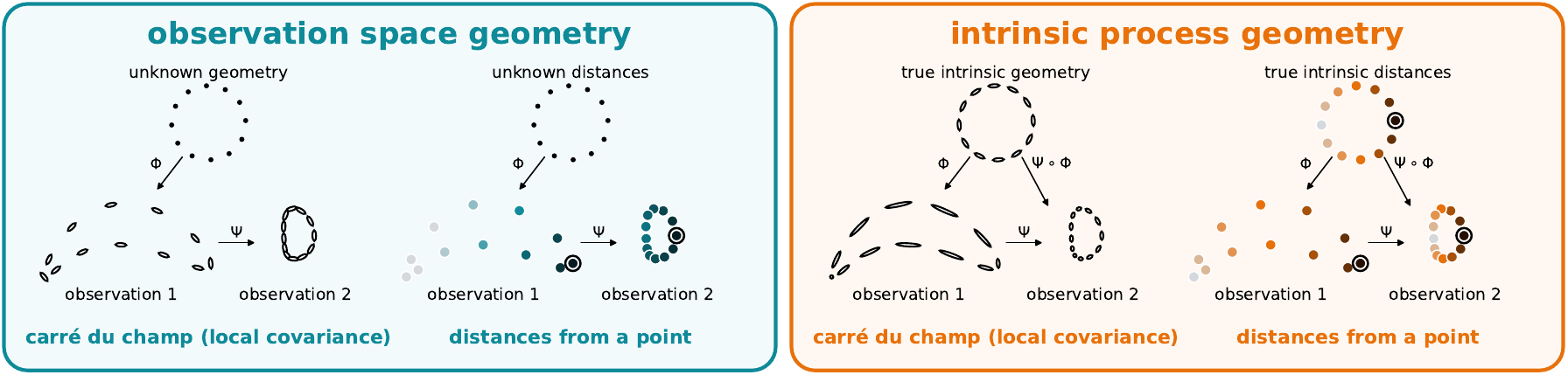}
    \caption{
    \textbf{Observation equivariance.}
    Standard geometric tools give data the \textcolor{gaussiancolor}{\textbf{geometry of the observation space (left)}}, so depend on the observation map.
    This affects the carré du champ, which measures local covariance and describes the data geometry, like the distances.
    These quantities cannot be reconciled across different observations.
    We learn the \textcolor{transitioncolor}{\textbf{intrinsic geometry of the underlying process (right)}}.
    The carré du champ measures intrinsic covariance and is observation-equivariant: it transforms naturally when we change observation.
    Distances are intrinsic and consistent.}
    \label{fig:illustration_fig_1}
    \vspace{-1.2em}
\end{figure}

To properly model this heterogeneous structure, we need a geometric theory that is both more flexible and computable from data.
These requirements point to \textbf{diffusion geometry} \citep{jones2024diffusiongeometry, bakry2013analysis}, which extends Riemannian constructions to irregular manifold-like structures via the calculus of Markov diffusion operators.
Since it is defined in terms of a process, it provides a natural model for the \textit{process geometry} underlying observed data.
Its metric-like analogue is the carré du champ (CDC) operator \citep{bakry1985diffusions, kunita1969absolute}, which measures the local directions of variability within the space, and can be used to recover its calculus and geometry.
Crucially, the CDC transforms naturally under observation: when an underlying, latent process is observed by a map into $\R^D$, the CDC pushes forward to $\R^D$ by conjugation with the derivative of the map.
We call this relationship \textbf{observation equivariance}: the geometry is still intrinsic, and the observation map can be exchanged for another in a predictable way (\cref{fig:illustration_fig_1}, right).
If we could find a way to learn this natural CDC from data, we would have successfully disentangled the process geometry from the observation geometry via a clean algebraic factorisation.

Standard geometric methods treat the data as an unordered point cloud with no additional structure, and equip it with the geometry of the ambient space it is observed in \citep{lebanon2012learning, hauberg2012geometric, jones2026computingdiffusiongeometry}.
On a manifold, this means using the induced ambient Riemannian metric, so geometric measurements, like geodesic distances, are observation-dependent (\cref{fig:illustration_fig_1}, left).

While this problem can be partially addressed with graph-based methods when the Jacobian of the observation map is known or can be estimated from dense local samples \citep{singer2008non, singer2009detecting,talmon2013empirical}, this information is rare in practice.
Graph methods also suffer from drawbacks like poor scalability and generalisation, and the curse of dimensionality, which can be overcome with deep learning.

Our approach overcomes these challenges by exploiting the fact that many real data sets have additional semantic structure in the form of natural \textit{perturbations} or \textit{augmentations}, which define their diffusion geometry.
For example, the steps in a time series represent a \textit{genuine sample from the underlying process}.
We show that the true CDC, and hence the true geometry, can be inferred from these perturbations, and provide a scalable algorithm for training a neural network to learn it.

\paragraph{Main contributions:}
\begin{enumerate}[label={C\arabic*.}, nosep]
    \item When data come from an unknown, latent stochastic process, we describe the intrinsic geometry in terms of the carré du champ operator (CDC).
    This CDC is \textit{intrinsic to the process} and \textit{observation-equivariant}.
    \item  We present \textbf{Intrinsic Geometry Matching (IGM)}: a scalable loss for learning the intrinsic CDC from data perturbations, and show that it is the unique observation-invariant loss.
    \item We exploit observation equivariance to train IGM from data in different spaces.
    \item We evaluate IGM on a range of high-dimensional datasets, demonstrating better performance, scaling properties, and training stability than existing methods.
\end{enumerate}

\section{Background}
\label{sec:background}

\subsection{Theories of differential geometry}

Let $\mathcal M$ be a smooth $d$-dimensional manifold. 
A \textbf{Riemannian metric} $g$ assigns, at each point $x\in\mathcal M$, an inner product $g_x$ on the tangent space $T_x\mathcal M$.
Its inverse $g_x^{-1}$ is an inner product on the cotangent space $T^*_x\mathcal M$ called the \textbf{cometric}.
These allow calculus objects to be defined on manifolds with curvature.
If $\Psi: \mathcal{M} \to \mathcal{N}$ is a smooth map with derivative $D_x\Psi$ at $x\in\mathcal{M}$, we can use $D_x\Psi$ to transport the metric and cometric back and forth between $\mathcal{M}$ and $\mathcal{N}$.
The metric $g_\mathcal{N}$ \emph{pulls back} to $(D_x\Psi)^\top g_\mathcal{N}(\Psi(x))\,(D_x\Psi)$, and the cometric $g_\mathcal{M}^{-1}$ \emph{pushes forward} to $(D_x\Psi) g^{-1}(x)\,(D_x\Psi)^\top$. 

\textbf{Diffusion geometry} extends Riemannian geometry to non-manifolds using the Bakry--\'Emery theory of Markov diffusion operators \citep{jones2024diffusiongeometry, bakry1985diffusions}. 
On a manifold $\mathcal{M}$, the geometry and calculus are captured by the metric, which satisfies
$
2g(\nabla f, \nabla h) = \Delta(fh) - f\Delta h - h \Delta f
$,
where $\Delta$ is the Laplace-Beltrami operator.
Since $\Delta$ is the generator of Brownian motion on $\mathcal{M}$, this identity relates the geometry to a diffusion process.
Diffusion geometry exploits this relationship to define geometry and calculus in more general settings, replacing the Brownian motion with a general Markov diffusion.
Let $(X_\tau)_{t\geq0}$ be a Markov process on a measure space $\mathcal X$.
If $f:\mathcal{X}\to\R$, its \textit{transition semigroup} is the operator $f \mapsto P_t f(x) := \mathbb E[f(X_\tau)\mid X_0=x]$, and its \textit{infinitesimal generator} is $\mathcal L f(x) := \lim_{t\downarrow0} t^{-1}(P_t f(x)-f(x))$. 
We define the {\bf \emph{carr\'e du champ} (CDC) operator} 
\begin{align}
\Gamma_{\mathcal L}(f,h)(x) 
\label{eq:cdc-definition}
&:= \tfrac12\left( \mathcal L(fh)(x) -f(x)\mathcal Lh(x) -h(x)\mathcal Lf(x) \right) \\
\label{eq:cdc-cov-definition}
&= \lim_{\tau\to0}\tfrac{1}{2\tau} \mathbb E[(f(X_\tau) - f(X_0))(h(X_\tau) - h(X_0))\mid X_0=x],
\end{align}
so that, when $\mathcal{L} = \Delta$ on a manifold, $\Gamma_\Delta(f,h) = g(\nabla f, \nabla h)$ \citep{kunita1969absolute,bakry1985diffusions}.
In local coordinates $x_i$, the matrix $[\Gamma_\mathcal{L}(x_i,x_j)]_{ij} = [g(\mathrm{d}x_i,\mathrm{d}x_j)]_{ij}$ is exactly the \textit{cometric}.
For diffusion processes it satisfies the \textit{diffusion property} (chain rule): if $\Phi$ is differentiable, then
$\Gamma_{\mathcal L}\!\left(\Phi(f_1,\ldots,f_k),h\right) = \sum_{i=1}^k \partial_i\Phi(f_1,\ldots,f_k)\, \Gamma_{\mathcal L}(f_i,h)$  \citep{bakry2013analysis}.

\subsection{Forward and backward Markov processes}

\label{subsec:fbtrf}

The methods we discuss here all use a notion of local kernel or perturbation, which can be unified in the language of \textit{Markov processes}.
Let $\mu$ be a probability measure on a state space $\mathcal X$, and let $P^+(x,\mathrm{d}y)$ be a \textbf{Markov kernel}: a probability measure for perturbations $y$ of a sample $x$.
We call $P^+$ the \textbf{forward process}.
If $X\sim\mu$ and $Y \sim P^+(X,\cdot)$, this defines a joint distribution $\pi(\mathrm{d}x,\mathrm{d}y)$ and marginal $\nu(\mathrm{d}y)$ for $Y$.
Together, these define the \textbf{backward process} $P^-(y, \mathrm{d}x)$ by
\(
\mu(\mathrm{d}x)\,P^+(x,\mathrm{d}y) = \pi(\mathrm{d}x,\mathrm{d}y) = \nu(\mathrm{d}y)\,P^-(y,\mathrm{d}x),
\)
which represents the conditional distribution of $X$ given $Y=y$, so $P^-$ is the $\mu$-dependent \textit{posterior} given the \textit{prior} $P^+$.

\subsection{Computing diffusion geometry}

Classical geometric methods use \textbf{graphs} and \textbf{kernels} to measure the affinity between points.
For example, if $k_\tau$ is a symmetric, positive kernel with bandwidth $\tau$, then the weighted graph Laplacian of \textit{Laplacian eigenmaps} \citep{belkin2003laplacian} and \textit{diffusion maps} \citep{coifman2006} is given by
\(
\Delta_\tau(f)(y)
= \tau^{-1} \sum_{i=1}^N
p_\tau(x_i,y)(f(y) - f(x_i))
\)
where $p_\tau(x,y) = k_\tau(x,y)/\sum_{j=1}^N k_\tau(x_j, y)$.
On a manifold with measure $\mu = \rho \mathrm{d}\mathrm{vol}$, $\Delta_\tau(f)$ converges to $\Delta(f) - g(\nabla 
\log\rho, \nabla f)$ as $\tau \to 0$ \citep{coifman2006}.
Laplacian estimates can be used to estimate the CDC: if $\Delta_\tau(f)$ is plugged into \eqref{eq:cdc-definition}, we obtain the covariance-like formula
\begin{equation}
\label{eq:poseterior_cdc_formula}
\Gamma_\tau(f,h)(y) 
= \frac{1}{2\tau} \sum_{i=1}^N p_\tau(x_i,y)(f(y) - f(x_i)) (h(y) - h(x_i)),
\end{equation}
with convergence $\Gamma_\tau(f,h) \to \Gamma(f,h)$ as $\tau \to 0$.
This has been used for kernel-based computation of diffusion geometry \citep{jones2024diffusiongeometry, jones2026manifold, berry2020spectral}.
For a more accurate and robust estimator, we can estimate the CDC as the true covariance
\(
\Gamma_\tau(f,h)(y) 
=  (2\tau)^{-1} \text{Cov}[ f(X), h(X) : X \sim  p_\tau(\cdot, y)]
\)
\citep{bamberger2026carre,jones2026computingdiffusiongeometry}.
However, these kernel-based approaches have all the drawbacks of kernel machine learning: poor scalability, the curse of dimensionality, and poor generalisation. 
This has motivated a range of new geometric methods based on \textbf{deep learning}.
Neural networks can be trained to estimate kernel-based expressions by recognising that kernel functions, like the Gaussian, are \textit{likelihood functions}, whose closed-form expressions can be replaced by Monte Carlo averages over samples.
By finding an appropriate sample-wise loss, we can estimate the same objects without ever evaluating the kernel.

This idea is central to diffusion models, where \textit{score matching} \citep{song2020score, song2019generative} is a kernel-free estimate of \textit{Laplacian eigenmaps}.
In this special case of Gaussian noise, we can recover the geometry for free, since the covariance-based CDC formula above is just the Jacobian of the score matching denoiser \citep{manor2024posterior, kharitenko2026landingscoreriemannianoptimization}, and is also related to the second-order score \citep{meng2021estimating}.
The first method to learn the CDC in full generality was \textit{Riemannian metric matching (RMM)} \citep{bamberger2026riemannianmetricmatchingscalable}, which can target either of the CDC estimates above, and exploits a low-rank parametrisation to allow scaling to high-dimensional spaces. \looseness=-1

\section{Intrinsic geometry of Observed Processes}

\label{sec:process-geometry}

Geometric methods involve a kernel or prior, and make estimates involving the posterior.
We ask: if the prior is more naturally adapted to the data, is the learned geometry more meaningful?
In the extreme: if the prior is genuinely \textit{intrinsic}, like sampling steps in a time series, can we learn the true intrinsic geometry?
We now define what `intrinsic geometry' means for data.

\label{subsec:naturality}

We will use the language of diffusion geometry \citep{jones2024diffusiongeometry}, and model the underlying process geometry as a diffusion Markov triple $(\mathcal X, \mu, \Gamma)$, where $\mathcal X$ is a measure space with measure $\mu$, and $\Gamma$ is the carré du champ (CDC) operator \citep{bakry2013analysis}. 
We make observations in $\R^D$ via some map $\Phi : \mathcal X \to \R^D$, so the data we see are samples $X = \Phi(\alpha)$ for $\alpha \sim \mu$, and we hope to recover the underlying process geometry.
However, the observation map $\Phi$ is typically arbitrary: it is not an \textit{isometric embedding} into $\R^D$, and can distort the structure.

\paragraph{Recovering the intrinsic CDC in observation space.}
Most geometric methods treat the observations as the \textit{only information available}, and equip it with the \textbf{geometry of the ambient space} $\R^D$.
For manifold data, this means giving the observed manifold the \textit{induced} Riemannian metric (the dot product in $\R^D$).
This geometry depends on $\Phi$, and is only correct when $\Phi$ is isometric.
Alternatively, there is a \textbf{natural geometry} available (\cref{fig:illustration_fig_1}).
In the manifold case, this is given by the pushforward cometric $(D_x\Phi) g^{-1}_x (D_x\Phi)^\top$.
Given a Markov triple, we can define the \textit{pushforward CDC} on $\Phi_\sharp \mu$.
If $\Phi$ is injective and if $f,h : \R^D \to \R$ and $x = \Phi(\alpha)$, then
\begin{equation}
\label{eq:pushforward_cdc_def_injective}
\Phi_\sharp \Gamma(f,h) (x) 
:= \Gamma(\Phi^*f, \Phi^*h)(\alpha),
\end{equation}
where $\Phi^*f = f \circ \Phi$ denotes the pullback.
If $\Phi$ is not injective, so  there are multiple $x$ for which $y = \Phi(x)$, we can extend the definition to a \textit{posterior average} over $x$, i.e.
\begin{equation}
\label{eq:pushforward_cdc_def_non_injective}
\Phi_\sharp \Gamma(f,h) (x) := \mathbb{E}\big[\Gamma(\Phi^*f, \Phi^*h)(\alpha) \mid \Phi(\alpha) = x \big]
\end{equation}
\citep[Sec. 1.15.4]{bakry2013analysis}.
Even though we remain in the observation space $\R^D$, with no knowledge of $\mathcal X$ or $\Phi$, we can use $\Phi_\sharp \Gamma$ to still work with the true geometry.
For example, if geodesics (shortest paths) are computed in $\R^D$ using $\Phi_\sharp \Gamma$, they will be the pushforward of the true geodesics on $\mathcal{X}$.
This definition is compositional: if $\Psi : \R^D \to \R^{D'}$ is differentiable, then
\(
(\Psi\circ\Phi)_\sharp \Gamma
= \Psi_\sharp\big[\Phi_\sharp \Gamma \big].
\)

\paragraph{The CDC of coordinate functions is sufficient and observation-equivariant.}
At a point $y \in \R^D$, let $G_\Phi(y)$ be the $D \times D$ matrix of the CDC of the coordinate functions $x_i$, so
\begin{equation}
\label{eq:G_matrix_def}
G_\Phi(x) 
= [\Phi_\sharp \Gamma(x_i, x_j)(x)]_{ij} 
= \mathbb{E}\big[\Gamma(\Phi_i, \Phi_j)(\alpha) \mid \Phi(\alpha) = x \big]_{ij}.
\end{equation}
Notice that $G_\Phi$ is \textit{sufficient to describe the whole CDC}, and hence all the geometry: if $f,h : \R^D \to \R$, then $\Phi_\sharp\Gamma(f,h)(x) = (D_xf) G_\Phi(x) (D_xh)^\top$.
If the observation map $\Phi$ is replaced by the composition $\Psi \circ \Phi$ with an injection $\Psi$, and $y = \Psi(x)$ then the chain rule for $\Gamma$ gives
\begin{equation}
\label{eq:conjugation_pushfoward_cdc}
G_{\Psi\circ\Phi}(y) 
= (D_x\Psi) G_\Phi(x) (D_x\Psi)^\top,
\end{equation}
a property we call \textbf{observation equivariance}: the CDC transforms naturally under changes in observation.
This is desirable because it means that the geometry is measuring \textit{only something intrinsic to the system}, independent of how it was observed.
This defines the following organising principle.

\begin{principle}
\textbf{Principle 1: intrinsic geometry is observation-equivariant.}
Observations are usually non-canonical, so the learned geometry should be \textit{observation-equivariant}.
Equivalently, the CDC should be the pushforward of the \textit{intrinsic} CDC to the observation space.
\end{principle}

\paragraph{Approximating $G_\Phi$ with the Markov process.}

This ideal CDC $G_\Phi$ is defined as the infinitesimal limit of a Markov process as time $\tau \to 0$ (\eqref{eq:cdc-cov-definition}), and our method will approximate this limit at small $\tau > 0$.
Let $(\alpha_\tau)_{\tau\ge0}$ be a Markov process on $\mathcal{X}$ with CDC $\Gamma$, so we observe the \textbf{pushforward process} $X_\tau = \Phi(\alpha_\tau)$ with CDC $\Phi_\sharp \Gamma$.
We define a covariance-style estimator
\begin{equation}
\label{eq:G_matrix_process}
G_{\tau}(x)
= \tfrac{1}{2\tau}\mathbb{E}
\big[(X_\tau^i-X_0^i)(X_\tau^j-X_0^j)
\mid X_0 = x \big]_{ij},
\end{equation}
so $\lim_{\tau\downarrow0} G_{\tau} = G_\Phi$.
We can also estimate $G_\Phi$ as the standard mean-centred covariance
\begin{equation}
\label{eq:G_matrix_cov_process}
G_{\tau,\mathrm{Cov}}(x)
= \tfrac{1}{2\tau}\mathrm{Cov}
\big[X_\tau^i,X_\tau^j
\mid X_0 = x \big]_{ij}
= \tfrac{1}{2\tau}\mathbb{E}
\big[(X_\tau^i-\mathbb{E}[X_\tau]^i)(X_\tau^j-\mathbb{E}[X_\tau]^j)
\mid X_0 = x \big]_{ij}.
\end{equation}
If the underlying diffusion process has a \textit{drift} $b$, then the mean $\mathbb{E}[X_\tau | X_0 = x]$ learns $X_0 + \tau b$, and centring removes the drift bias from the estimate and reduces the variance of the CBMM loss.
By approximating either $G_\tau$ or $G_{\tau,\mathrm{Cov}}$, we can estimate $G_\Phi$ and recover the full CDC $\Phi_\sharp \Gamma$.

\section{Intrinsic Geometry Matching}

\label{sec:igm}

\subsection{Conditional training with Bregman divergences}
\label{subsec:igm-cond-training}


We now describe a procedure for training a neural network $G^\theta_\tau$ to learn $G_\tau$ or $G_{\tau,\mathrm{Cov}}$.
Suppose we have samples of data $X \sim \mu$ and perturbations $Y \sim P^\tau(X, \mathrm{d}y)$.
To train $G_\tau^\theta$, we need to pick a matrix loss function, and a natural choice is a \textit{Bregman divergence}.
\begin{definition} \label{def:bregman_div}
Let $\mathcal A \subseteq \mathbb S^D$ be convex and $\omega: \mathcal A \to\mathbb R$ be strictly convex and differentiable.
The \textbf{Bregman divergence} on $\mathcal{A}$ induced by $\omega$ is $\mathcal{D}_\omega(A,B) = \omega(A)-\omega(B) -\left\langle\nabla\omega(B),A-B\right\rangle$.
\end{definition}

For example, the Frobenius loss $\|A - B\|^2$ is the Bregman divergence induced by $\omega(X) = \|X\|^2$.
Ideally, we would pick $\mathcal{D}(\cdot, \cdot)$, and regress $G_\tau^\theta$ over $\tau \sim \mathcal{U}[0,\tau_\text{max}]$ and $X\sim\mu$ with
\[
\mathcal L_{\mathrm{BMM}}(\theta)  = \mathbb{E}_{\tau, X} \big[ \mathcal{D}(G_\tau(X), G_\tau^\theta(X)) \big],
\]
which we call the \textbf{Bregman metric matching (BMM)} loss.
Evaluating this loss is impossible as we do not know $G_\tau$ in advance, so we introduce the tractable alternative using $Y \sim P^\tau(X, \mathrm{d}y)$
\[
\mathcal L_{\mathrm{CBMM}}(\theta) = \mathbb{E}_{\tau, X, Y \mid X} \big[ \mathcal{D}(\tfrac{1}{2\tau}(Y-X)(Y-X)^\top, G_\tau^\theta(X)) \big],
\]
which we call the \textbf{conditional Bregman metric matching (CBMM)} loss.
The CBMM loss can be cheaply evaluated given only samples $X$ and $Y$, and is trivially parallelisable over those samples for distributed training.
Crucially, the simple and tractable $\mathcal L_{\mathrm{CBMM}}$ is equal to $\mathcal L_{\mathrm{BMM}}$ up to a constant independent of $\theta$, so $\nabla_\theta \mathcal L_{\mathrm{CBMM}} = \nabla_\theta \mathcal L_{\mathrm{BMM}}$ and training with either loss is equivalent.

\begin{proposition}[Conditional Bregman metric matching]
\label{prop:cbmm}
We have, for sufficiently regular $\mathcal{D}$,
\begin{equation}
\label{eq:bmm-cbmm-equivalence}
\mathcal L_{\mathrm{CBMM}}(\theta) = \mathcal L_{\mathrm{BMM}}(\theta) + \mathbb E\!\left[\mathcal{D}(\tfrac{1}{2\tau}(Y-X)(Y-X)^\top,G_\tau(X))\right]
\end{equation}
if and only if $\mathcal{D}$ is a Bregman divergence \citep[follows from Theorems 1 \& 4]{banerjee2005optimality}.
\end{proposition}

This shows that we can learn $G_\tau$ using \textit{only samples from $P^\tau$}, without prior knowledge of a kernel function, so can apply our method to `black box' data-driven perturbations.
When $P^\tau$ represents intrinsic transitions of an unknown process, we can still recover the intrinsic geometry.

\begin{principle}
\textbf{Principle 2: learn the intrinsic geometry from data-driven perturbations.}
Data-driven perturbations (like steps in a time series) represent \textit{oracle access to the intrinsic process}, and can be used to learn the intrinsic geometry despite the arbitrary observation.
\end{principle}

\paragraph{Mean-centred covariance.}

As stated, CBMM learns the uncentred covariance $G_\tau$.
To instead learn the mean-centred covariance $G_{\tau, \mathrm{Cov}}$, we need a model $m_\tau^\phi(x)$ for the mean $\mathbb{E}[X_\tau \mid X_0 = x \big]$, which can be trained with
$
\mathcal L(\phi) = \mathbb{E}_{\tau, X, Y \mid X} \| Y - m_\tau^\phi(X) \|^2.
$
We then define the conditional BMM loss
\[
\mathcal L_{\mathrm{CBMM}}^{\mathrm{Cov}}(\theta) = \mathbb{E}_{\tau, X, Y \mid X} \big[ \mathcal{D}(\tfrac{1}{2\tau}(Y-m_\tau^\phi(X))(Y-m_\tau^\phi(X))^\top, G_\tau^\theta(X)) \big],
\]
which, by the same argument as above, is equivalent to $\mathbb{E}_{\tau, X} [ \mathcal{D}(G_{\tau,\mathrm{Cov}}(X), G_\tau^\theta(X))]$.

\paragraph{Learning the backward process CDC.}
We have defined the BMM and CBMM losses as estimators for the forward process CDC $G_\tau$, which is natural when $P^\tau$ gives samples from an underlying Markov process.
However, most methods for computational geometry and generative modelling use an uninformative forward kernel $P^\tau$ (e.g. Gaussian), and the posterior $P^{-\tau}$ models the structure.
We can learn the backward CDC $G_{-\tau}$ by simply replacing $X \sim \mu$ with $Y \sim \nu$ (\cref{subsec:fbtrf}) in $L_{\mathrm{BMM}}$, and $G_\tau(X)$ with $G_\tau(Y)$ in $L_{\mathrm{CBMM}}$.
In this case we recover the RMM loss of \cite{bamberger2026riemannianmetricmatchingscalable} when $Y\mid X=x\sim\mathcal{N}(x,\tau I)$ and $\mathcal{D}(\cdot,\cdot)$ is Frobenius.

\subsection{LogDet is the natural Bregman divergence for the CDC}
\label{subsec:logdet_naturality}

We can use $L_{\mathrm{CBMM}}$ with any Bregman divergence $\mathcal{D}$ to learn $G_\tau$.
Since $\mathcal{D}$ is defined on the image of $\Phi$, which we presume to be arbitrary, we can ask whether there is a $\mathcal{D}$ that is itself observation-invariant.
If $J$ is a $D \times D$ matrix representing a linear change in observation, then $J$ transforms the CDC by conjugation, so an ideal $\mathcal{D}$ would satisfy $\mathcal{D}(JAJ^\top,JBJ^\top)=\mathcal{D}(A,B)$.
If $\mathcal{D}(A,B) = \|A-B\|^2$, this holds only when $J$ is orthogonal.
For general $J$, this condition specifies $\mathcal{D}$ uniquely.

\begin{wrapfigure}{r}{0.3\linewidth}
    \vspace{-1em}
    \centering
    \includegraphics[width=1\linewidth]{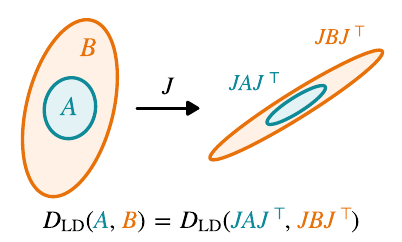}
    \vspace{-2em}
    \caption{Observation invariance of the LogDet divergence.}
    \label{fig:identifiability}
    \vspace{-1em}
\end{wrapfigure}

\begin{restatable}[Naturality of the LogDet divergence]{proposition}{logdetnaturality}
\label{prop:ld-naturality}
Let $\mathcal{D}$ be a Bregman divergence on the space of $D \times D$ positive definite matrices.
Then the following are equivalent:
\begin{enumerate}[nosep]
    \item $\mathcal{D}(JAJ^\top,JBJ^\top)=\mathcal{D}(A,B)$ for all
    $A,B\succ0$ and $J\in\mathrm{GL}(D)$;
    \item $\mathcal{D}(A,B) = \operatorname{tr}(B^{-1}A) - \log\det(B^{-1}A) - D$, up to scale.
\end{enumerate}
\end{restatable}

This divergence is called the \textbf{LogDet} $\mathcal{D}_{\mathrm{LD}}$, and is a popular loss for matrix regression \citep{davis2007information, kulis2009low,james1961estimation}.
An apparent problem is that $\mathcal L_{\mathrm{CBMM}}$ uses rank-1 targets which are not positive definite.
However, up to a constant independent of $B$, $\mathcal{D}_{\mathrm{LD}}(A,B) = \operatorname{tr}(B^{-1}A) + \log\det(B)$ is well-defined, so \cref{prop:cbmm} still applies.
If $\Psi:\R^D \to \R^D$ is an arbitrary injective and differentiable map, representing a non-linear change of observation, then $\Phi$ acts on the CDC by conjugation with the $D\times D$ matrix $D\Psi$.
The conjugation-invariance of LogDet then gives asymptotic \textbf{non-linear observation-invariance} to $\Psi$ in the limit $\tau\to0$.

\begin{restatable}
{corollary}{nonlinldnat}
\label{cor:ld-nonlinear-naturality}
Let $\Psi:\R^D\to\R^D$ be a local diffeomorphism at $x$.
Then, for any $G^\theta(x)\succ0$,
\[
\mathcal{D}_{\mathrm{LD}}\!\left( G_{\Psi\circ\Phi}(\Psi(x)),(D_x\Psi)G^\theta(x)(D_x\Psi)^\top \right) = \mathcal{D}_{\mathrm{LD}}\!\left( G_\Phi(x),G^\theta(x) \right).
\]
\end{restatable}

The proofs are in App.~\ref{app:proofs}.
These results show that LogDet is the \textit{only observation-invariant loss that permits sample-wise training}.
We therefore call LogDet-CBMM \textbf{intrinsic geometry matching (IGM)}: both the loss and its target depend only on the process and not on its representation in $\R^D$.

\begin{principle}
\textbf{Principle 3: LogDet is the natural loss for learning geometry.}
It is the only loss that permits sample-wise training and is observation-invariant, so \textit{the training itself} is intrinsic.
\end{principle}

\paragraph{LogDet is just Gaussian NLL.} 

If $m^{\phi}_\tau(X)$ models the conditional mean $\mathbb{E}[Y\mid X,\tau]$, then, neglecting terms independent of $\phi$ and $\theta$, we can write the mean-centred CBMM loss as
\[
\mathcal L_{\mathrm{CBMM}}(\phi,\theta) = \mathbb{E}_{\tau, X, Y \mid X} \big[ \tfrac{1}{2\tau}(Y - m^{\phi}_\tau(X))^\top G^\theta_\tau(X)^{-1}(Y - m^{\phi}_\tau(X)) + \log \det G^\theta_\tau(X)
 \big].
\]
Up to adding the Gaussian normalisation constant and dividing by 2, this is the \textbf{negative log-likelihood (NLL)} of $Y$ under $\mathcal{N}(m^{\phi}_\tau(X),2\tau G^\theta_\tau(X))$\footnote{Without a mean model, the Gaussian is $\mathcal{N}(X,2\tau G^\theta_\tau(X))$.}.
It is minimised by the true mean and covariance \citep{dawid1999coherent}, so can be used to jointly train the parameters $\phi$ and $\theta$.

\paragraph{Low-rank parametrisation.} 
The $D\times D$ matrix $G^\theta_\tau$ is too large to materialise in memory for large dimension $D$.
If we parametrise $G^\theta = M^\theta(M^\theta)^\top + \epsilon^\theta I_D$, where $M^\theta$ is a $D \times r$ low-rank factor and $\epsilon^\theta$ is a learned scalar, then \textbf{LogDet can be evaluated using only the $r \times r$ matrix $(M^\theta)^\top M^\theta$}, allowing easy scaling to high $D$ (App. \ref{app:low-rank}).
We may prefer a low-rank model and set $\epsilon^\theta = 0$, but this breaks the strict positive definiteness required for LogDet: in this case we use the regularised loss $\mathcal{D}_{\mathrm{LD}}^\lambda(A,B) = \mathcal{D}_{\mathrm{LD}}(A + \lambda I, B + \lambda I)$ for a small $\lambda$, which retains the correct minimiser.
For a $d$-dimensional manifold, it suffices to take $r \ge 2d-1$ \citep[Thm. 4.4]{bamberger2026riemannianmetricmatchingscalable}.

\subsection{Training IGM with data in different spaces}
\label{subsec:diffspaces}

Observation equivariance explicitly relates the geometry of different spaces.
We can exploit this to train a model in one space with data from another, potentially unlocking opportunities for scaling data in applications.
There are two settings for this, based on the information available.

\paragraph{Pushing forward the CDC.}
Since the CDC pushes forward naturally, this direction is easy.
Suppose we have data $X$ and perturbations $Y \mid X$ in $\R^D$, and $\Psi:\R^D\to\R^{D'}$, which could represent an encoder into a latent space.
Then we can apply standard IGM to the pushfoward data $\Psi(X)$ and perturbations $\Psi(Y) \mid \Psi(X)$, which is exactly our standard setup of \cref{sec:process-geometry}:
we regress $G^\theta_\tau$ against rank-1 targets $\Psi(Y)-\Psi(X) \in \R^{D'}$.
If $X = \Phi(\alpha)$, for unknown latent $\alpha$, then we approximate the composed CDC $G_{\Psi\circ\Phi}$.
Given access to the derivative $D\Psi$, we can also train directly on the conjugated perturbation, giving
the rank-1 targets $(D_X\Psi)(Y-X) \in \R^{D'}$.

\paragraph{Recovering the original CDC.}
The harder case is for observed perturbations in $\R^{D'}$, but we want to learn the CDC on $\R^D$.
There are several possibilities, but here we assume we have $X \in \R^D$, the map $\Psi$ and its differential $D\Psi$, and perturbations $Z \mid \Psi(X)$ in $\R^{D'}$.
We can then formulate the inverse problem for $G_
\tau$ by matching the pushed forward CDC against the perturbations in $\R^{D'}$:
\begin{equation}
\label{eq:intrinsic-observation-cbmm}
\mathbb{E}_{\tau, X, Z \mid \Psi(X)} \big[ \mathcal{D}(\tfrac{1}{2\tau}(Z-\Psi(X))(Z-\Psi(X))^\top, (D_X\Psi)G_\tau^\theta(X)(D_X\Psi)^\top) \big].
\end{equation}
It may be that $\Psi$ collapses too much information to fully recover $G_\tau$, but may still provide useful additional training when perturbations in $\R^D$ are limited.
We could also combine multiple partial data sources to fully determine $G_\tau$: we discuss the formal identifiability criteria in App.~\ref{app:identifiability}.

\section{Experiments}


\subsection{Recovering process geometry}
\label{subsec:synthetic}

We test four defining properties of IGM in a controlled setting with known process geometry and observation map:
\hypertarget{i}{\textcolor{reflink}{\textbf{(i)}}} recovering the true process geometry under non-isometric observation;
\hypertarget{ii}{\textcolor{reflink}{\textbf{(ii)}}} the advantage of LogDet over Frobenius loss;
\hypertarget{iii}{\textcolor{reflink}{\textbf{(iii)}}} scaling laws, in terms of intrinsic and observation dimensions; and
\hypertarget{iv}{\textcolor{reflink}{\textbf{(iv)}}} recovering intrinsic geometry from an ensemble of partial, rank-deficient views.

We simulate a latent process as Brownian motion on the flat torus $\mathbb T^d$, with cometric $I_d$, observed through a curved, non-isometric embedding into $\R^D$ (\cref{fig:synthetic}\textcolor{reflink}{\textbf{(a)}}).
We train different CDCs on $\R^D$ with different methods: IGM uses the pushforward method with the known differential (\cref{subsec:diffspaces}), FGM (Frobenius geometry matching) denotes the same method but with Frobenius regression not LogDet, and LD-RMM (LogDet Riemannian metric matching)  uses ambient Gaussian perturbations with LogDet\footnote{LogDet-RMM modifies the RMM of \cite{bamberger2026riemannianmetricmatchingscalable} by changing the loss from Frobenius to LogDet and jointly training a mean and covariance model (which is not possible with Frobenius).}.
Finally, $k$NN estimates a centred, kernel-weighted covariance of clean observations \citep{jones2026computingdiffusiongeometry}.
Further details are in App.~\ref{app:synthetic}. \looseness =-1

\vspace{-20pt}
\newsavebox{\tblbox}
\newlength{\tblheight}
\savebox{\tblbox}{%
\small
\renewcommand{\arraystretch}{1.2}
\setlength{\tabcolsep}{1.5pt}
\begin{tabular}{@{}l *{4}{c}@{}}
\toprule
$d =$ & 2 & 8 & 16 & 32 \\
\midrule
\multicolumn{5}{@{}l}{\emph{Latent-cometric error $\mathcal{E}_\Gamma$}} \\
\addlinespace[1pt]
FGM
& $0.006$
& $0.286$
& $0.343$
& $0.390$ \\
\rowcolor{winrow}
\textbf{IGM}
& $\mathbf{0.002}$
& $\mathbf{0.005}$
& $\mathbf{0.012}$
& $\mathbf{0.029}$ \\
\addlinespace[3pt]

\multicolumn{5}{@{}l}{\emph{Tangent-projector error $\mathcal{E}_P$}} \\
\addlinespace[1pt]
FGM
& $1{\times}10^{-5}$
& $0.092$
& $0.140$
& $0.263$ \\
\rowcolor{winrow}
\textbf{IGM}
& $\mathbf{2{\times}10^{-7}}$
& $\mathbf{2{\times}10^{-6}}$
& $\mathbf{1{\times}10^{-5}}$
& $\mathbf{4{\times}10^{-5}}$ \\
\bottomrule
\end{tabular}%
}
\setlength{\tblheight}{\dimexpr\ht\tblbox+\dp\tblbox\relax}
\setlength{\tblheight}{\dimexpr\ht\tblbox+\dp\tblbox\relax }
\begin{figure}[h]
\centering

\begin{subfigure}[c]{0.33\textwidth}
\centering
\begin{minipage}[c][\tblheight][c]{0.46\linewidth}
    \centering
    \includegraphics[
        width=\linewidth,
        height=\tblheight,
        keepaspectratio
    ]{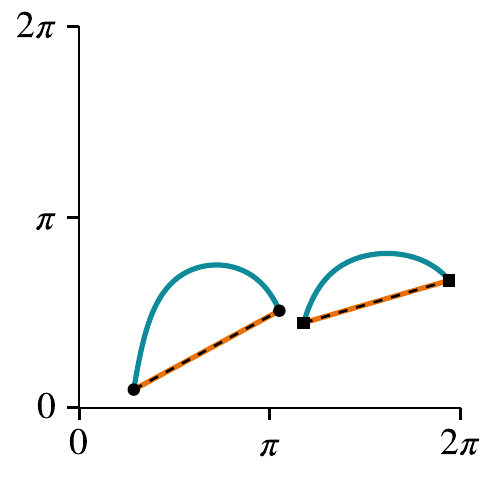}
\end{minipage}
\hfill
\begin{minipage}[c][\tblheight][c]{0.50\linewidth}
    \centering
    \includegraphics[
        width=\linewidth,
        height=\tblheight,
        keepaspectratio
    ]{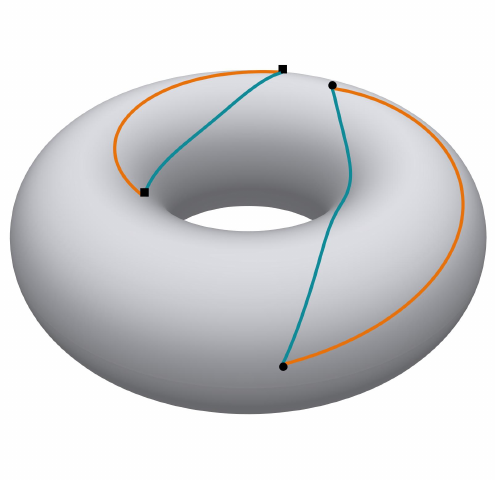}
\end{minipage}
\vspace{-20pt}
\caption*{\hypertarget{i}{\textcolor{reflink}{\textbf{(a)}}}}
\end{subfigure}
\hfill
%
\begin{subfigure}[c]{0.2\textwidth}
\centering
\begin{minipage}[c][\tblheight][c]{\linewidth}
\centering
\resizebox{\linewidth}{!}{\usebox{\tblbox}}
\end{minipage}
\vspace{-20pt}
\caption*{\textcolor{reflink}{\textbf{(b)}}}
\label{fig:synthetic-table}
\end{subfigure}
\hfill
\begin{subfigure}[c]{0.44\textwidth}
\centering
\begin{minipage}[c][\tblheight][c]{0.49\linewidth}
    \centering
    \includegraphics[
        width=\linewidth,
        height=\tblheight,
        keepaspectratio
    ]{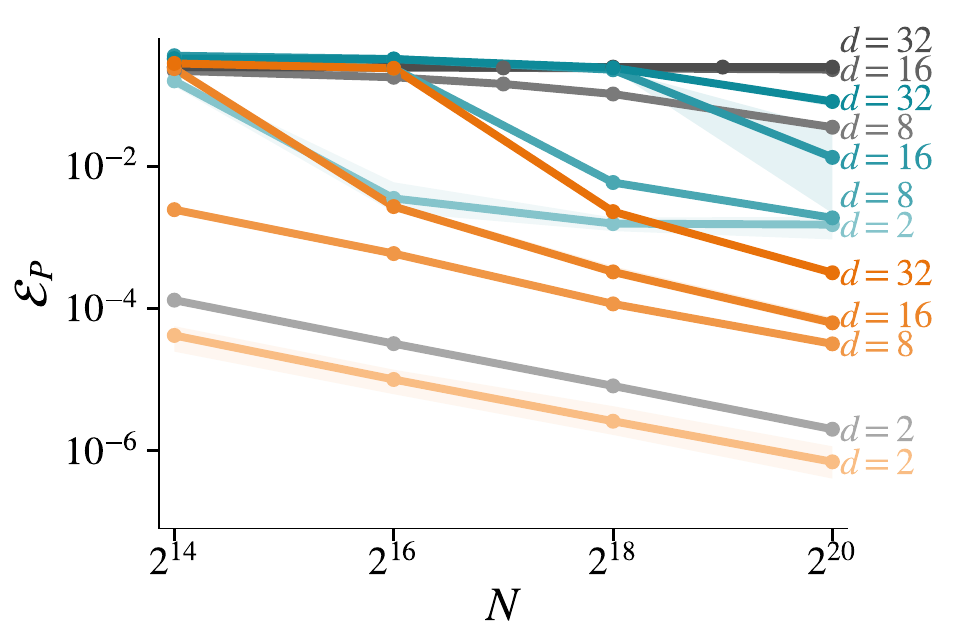}
\end{minipage}
\hfill
\begin{minipage}[c][\tblheight][c]{0.49\linewidth}
    \centering
    \includegraphics[
        width=\linewidth,
        height=\tblheight,
        keepaspectratio
    ]{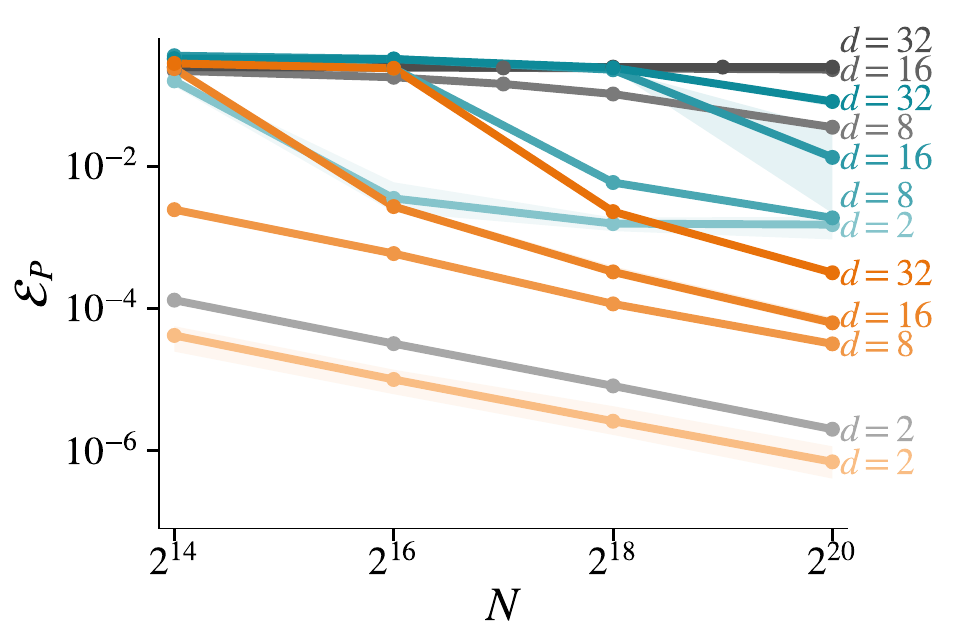}
\end{minipage}

\vspace{-20pt}
\caption*{\hypertarget{iii}{\textcolor{reflink}{\textbf{(c)}}}}
\end{subfigure}
\hfill
%

\vspace{-0.5em}
\caption{\textbf{IGM recovers process geometry with intrinsic-dimensional scaling.}
\textcolor{transitioncolor}{\textbf{Orange}}: IGM;
\textcolor{gaussiancolor}{\textbf{teal}}: LD-RMM;
\textcolor{knn}{\textbf{grey}}: $k$NN.
\textcolor{reflink}{\textbf{(a)}}
Geodesics in latent coordinates (left) and on the observed torus (right): IGM recovers the true geometry, while LD-RMM induces the observation geometry; markers ($\bullet$, $\blacksquare$) identify endpoint pairs.
\textcolor{reflink}{\textbf{(b)}}
Latent-cometric and tangent-projector errors across intrinsic dimension.
\textcolor{reflink}{\textbf{(c)}}
Tangent-projector error versus sample size, varying intrinsic (left) and observation dimension (right).
\looseness=-1}
\label{fig:synthetic}
\end{figure}
\vspace{-1.em}

\paragraph{\textcolor{reflink}{\textbf{(i)}}}
We compute \textbf{geodesics} in latent coordinates for each CDC by minimising the path energy (\cref{fig:synthetic}\textcolor{reflink}{\textbf{(a)}}).
IGM recovers the true, straight line geodesics (the shortest paths for the process), whereas LD-RMM finds the shortest paths in observation space.

\paragraph{\textcolor{reflink}{\textbf{(ii)}}}
To measure \textbf{intrinsic recovery}, we push forward the CDC to the torus through the inverse of the observation, and compare the pointwise Frobenius error to the true cometric $\mathcal E_\Gamma$.
We compare LogDet and Frobenius regression, for fixed $N=10^5$ points (\cref{fig:synthetic}\textcolor{reflink}{\textbf{(b)}}, top).
LogDet has a clear advantage, which becomes more pronounced as $d$ increases.

\paragraph{\textcolor{reflink}{\textbf{(iii)}}}
We compare qualitatively different CDCs using the fact that they encode the same tangent space: the top-$d$ eigenvectors at $x \in \mathcal{M} \subset \R^D$ should form an orthogonal projection onto $T_x\mathcal{M}$.
We measure \textbf{tangent recovery} by $\mathcal E_P$, the squared Frobenius distance to the true rank-$d$ tangent projector.
Training models on increasingly more data lets us compare each method's sample efficiency and how it depends on dimension.
In \cref{fig:synthetic}\textcolor{reflink}{\textbf{(c)}} we vary intrinsic dimension $d$ while fixing ambient dimension  $D=48$ \textcolor{reflink}{(left)}, and vary $D$ at fixed $d=8$ \textcolor{reflink}{(right)}, and compare IGM, LD-RMM, and $k$NN.
The curves report three-seed means of $\mathcal E_P$, and bands show the range.
Both neural methods outperform the kernel, as in \cite{bamberger2026riemannianmetricmatchingscalable}.
However, while LD-RMM infers the tangents from ambient noise, IGM is trained only on intrinsic perturbations which are tangential, giving a low-dimensional signal.
This makes IGM much more sample efficient as dimension increases.
The effect is especially pronounced when we vary only $D$ \textcolor{reflink}{(right)}, where IGM is essentially unaffected.

\begin{wrapfigure}{r}{0.18\linewidth}
    \vspace{-1em}
    \centering
    \includegraphics[width=1\linewidth]{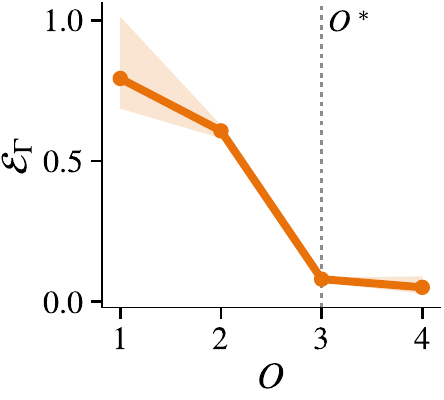}
    \vspace{-2em}
    \caption{Error vs number of views.}
    \label{fig:identifiability}
    \vspace{-2em}
\end{wrapfigure}

\paragraph{\textcolor{reflink}{\textbf{(iv)}}}
Following \cref{subsec:diffspaces}, we evaluate whether multiple, individually rank-deficient observations can be combined to fully recover the intrinsic CDC.
We use \textit{winding characters} $\psi_{\mathbf m}(x)=(\cos\langle\mathbf m,x\rangle,\sin\langle\mathbf m,x\rangle)$, which are observations $\mathbb{T}^2\to\R^2$, and each revealing only the rank-one quadratic form $\mathbf m^\top\Gamma\mathbf m$.
As we discuss in App. \ref{app:identifiability}, the true CDC is identifiable precisely when $\{\mathbf m_o\mathbf m_o^\top\}_{o=1}^O$ spans $\operatorname{Sym}(\mathbb R^d)$.
In this case, the critical threshold is $O^\star=3$, and we observe this empirically in \cref{fig:identifiability}.

\subsection{Image geometry through augmentation}
\label{sec:mnist}

When natural process transitions are unavailable, we can simulate them with \emph{data augmentation}.
This encompasses Gaussian noise models, where the prior contains no geometric information, so we now investigate whether semantically meaningful augmentations $Y\sim P^\tau(X,\cdot)$ induce a better geometry.
We test whether
\textcolor{reflink}{\textbf{(i)}} IGM on principled augmentations has better downstream performance than LD-RMM,
\textcolor{reflink}{\textbf{(ii)}} this process geometry transfers across data distributions without retraining.
We set $P^\tau$ to be a smooth spatial deformation in the plane, known as \textit{metamorphosis} \citep{holm2008eulerpoincaretheorymetamorphosis,younes2010shapes,franois2021metamorphicimageregistrationusing}.
IGM learns the uncentred forward field $G_\tau$ from these transitions, whereas LD-RMM learns a mean-centred backward field from isotropic Gaussian perturbations.  
Both covariance estimators use the same U-Net backbone and rank-$100$ factorization with a learned isotropic floor. 
Full details are in App.~\ref{app:mnist}. 

\newlength{\vidfigheight}
\setlength{\vidfigheight}{0.145\textwidth}

\begin{figure}[h]
\vspace{-0.5em}
\centering
\begin{subfigure}[c]{\linewidth}
    \centering

    \makebox[\linewidth][c]{%
        \includegraphics[width=0.48\linewidth]
            {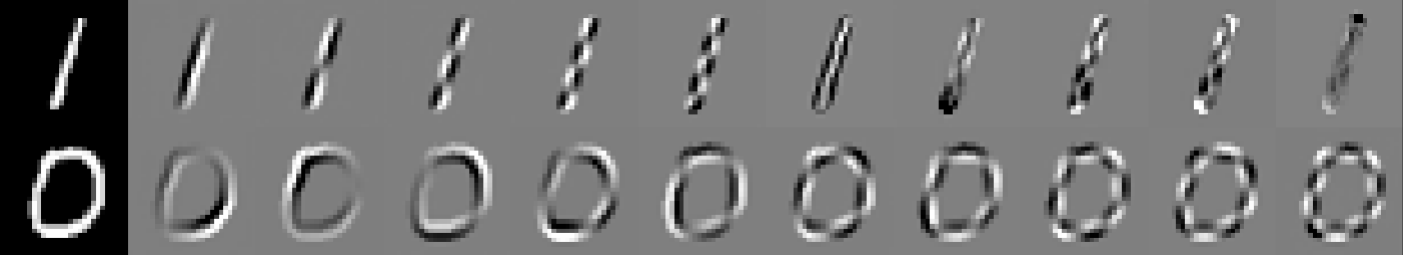}%
        \hspace{2pt}%
        \includegraphics[width=0.26\linewidth]
            {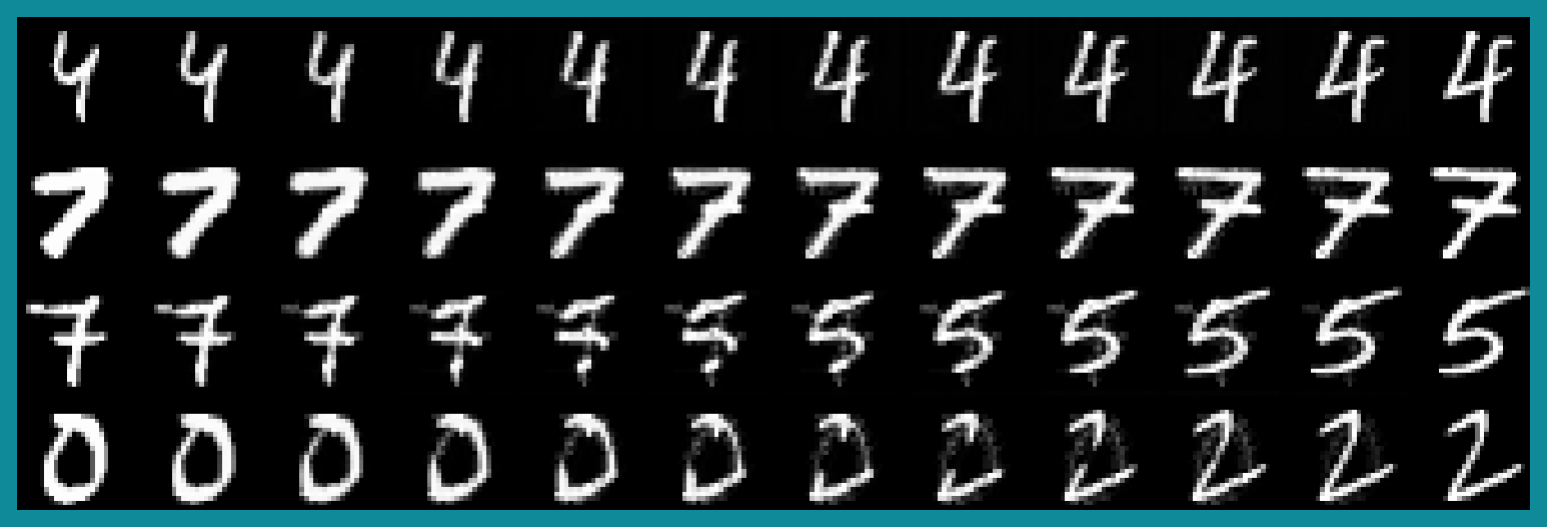}%
        \hspace{2pt}%
        \includegraphics[width=0.26\linewidth]
            {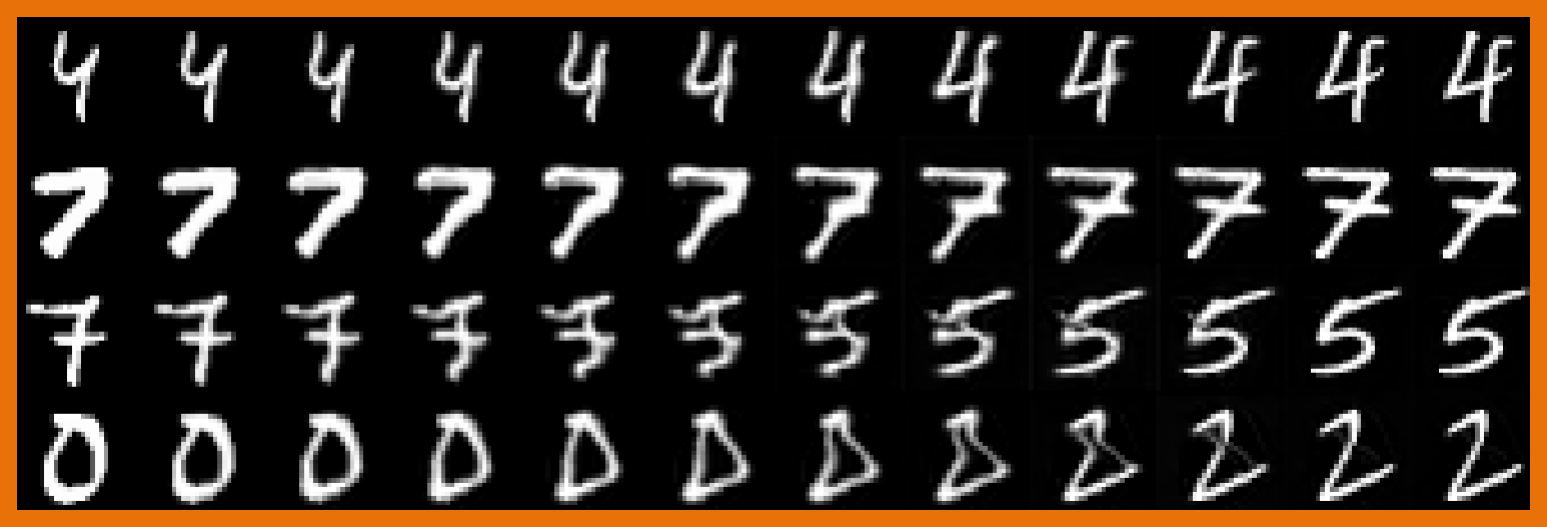}%
    }

    \makebox[\linewidth][c]{%
        \includegraphics[width=0.48\linewidth]
            {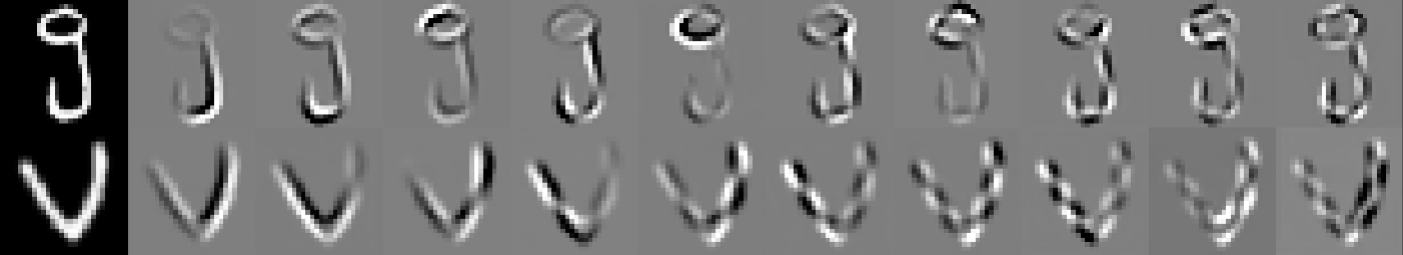}%
        \hspace{2pt}%
        \includegraphics[width=0.26\linewidth]
            {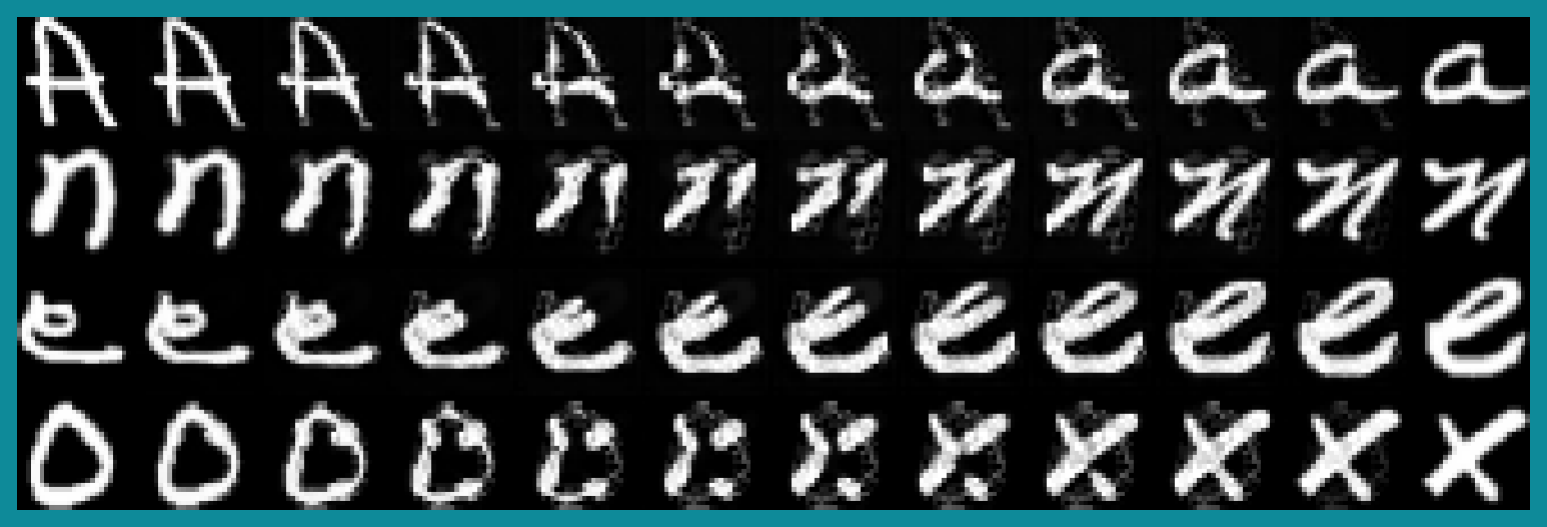}%
        \hspace{2pt}%
        \includegraphics[width=0.26\linewidth]
            {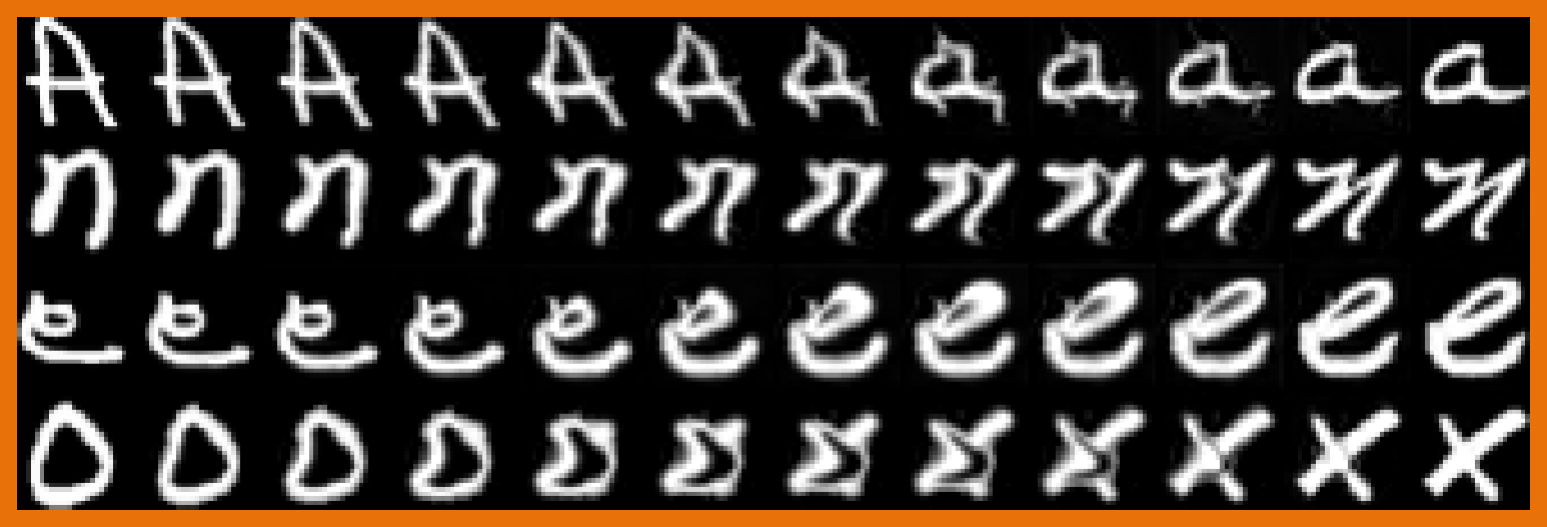}%
    }

    \makebox[\linewidth][c]{%
        \makebox[0.47\linewidth][c]{%
            \textcolor{reflink}{\textbf{(a)}}%
        }%
        \hspace{2pt}%
        \makebox[0.26\linewidth][c]{%
            \textcolor{reflink}{\textbf{(b)}}%
        }%
        \hspace{2pt}%
        \makebox[0.26\linewidth][c]{%
            \textcolor{reflink}{\textbf{(c)}}%
        }%
    }
\end{subfigure}
\hfill

\vspace{-0.7em}
\caption{\textbf{Process geometry from image transitions.}
\textcolor{reflink}{\textbf{(a)}} Leading CDC eigenvectors: metamorphic augmentations induce semantically meaningful tangent vectors.
\textcolor{reflink}{\textbf{(b)}}
Interpolating paths with LD-RMM (observation-space geometry).
\textcolor{reflink}{\textbf{(c)}}
Interpolating paths with IGM (intrinsic metamorphic geometry).
Models trained on MNIST (top row) are transferred to EMNIST (bottom row) without fine-tuning. Additional eigenvectors across scales and interpolation paths are shown in \cref{fig:emnist-scale,fig:additional-image-interpolations}, respectively.
}
\label{fig:image-geometry}
\end{figure}
\vspace{-0.8em}

\paragraph{\textcolor{reflink}{\textbf{(i)}}} 
We can visualise the learned CDC geometry from its leading eigenvectors, which generate the tangent space at each point.
These learned vectors are highly interpretable, and describe increasingly complex deformations.
We apply the CDC to compute interpolating paths between images using \textit{Riemannian optimisation}, taking gradient-descent steps from a source image $x_0$ down the target potential $\tfrac12\|x-x_1\|_2^2$.
This is a challenging problem in pixel space, and existing approaches typically use a per-step denoiser to improve path quality \citep{kharitenko2026landingscoreriemannianoptimization, kawasaki2026diffusion}.
We evaluate without these tricks, and find that IGM is sufficient for smooth, semantic transformations \cref{fig:image-geometry}\textcolor{reflink}{\textbf{(c)}} while LD-RMM produces irregular paths with pixel-level artifacts \cref{fig:image-geometry}\textcolor{reflink}{\textbf{(b)}}.

\paragraph{\textcolor{reflink}{\textbf{(ii)}}}
We perform the same evaluations on out-of-distribution EMNIST letters without fine-tuning.
The LD-RMM interpolating paths lose quality on this unseen data, while IGM has no visible degradation and retains meaningful tangent vectors.
This suggests that IGM geometry is \emph{portable}, and can robustly generalise beyond the training distribution.

\newpage
\subsection{Natural image geometry from videos}
\label{subsec:videos}

\begin{wrapfigure}{r}{0.4\linewidth}
    \vspace{-1em}
    \centering

    \begin{minipage}[t]{0.48\linewidth}
        \centering
        \includegraphics[width=\linewidth]
            {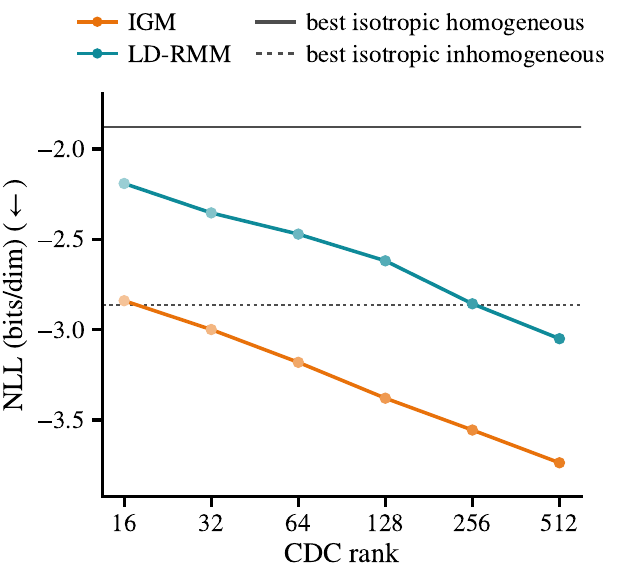}\par
        \vspace{0.15em}
        \makebox[\linewidth][l]{%
            \hspace*{0.58\linewidth}%
            \makebox[0pt][c]{\small\textcolor{reflink}{\textbf{(a)}}}%
        }
    \end{minipage}
    \hfill
    \begin{minipage}[t]{0.48\linewidth}
        \centering
        \includegraphics[width=\linewidth]
            {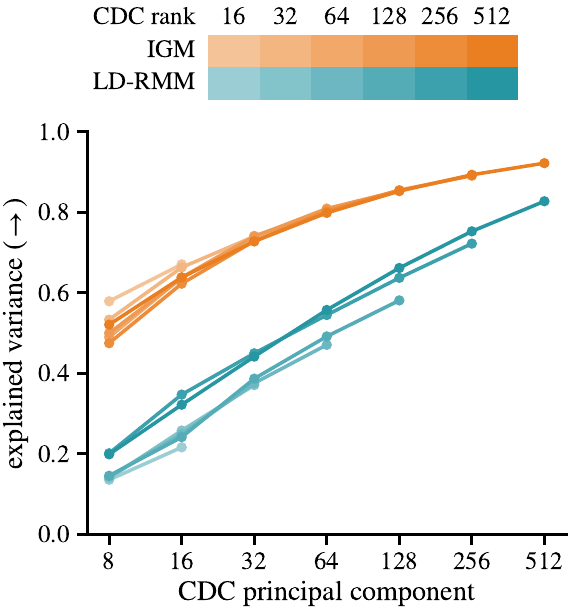}\par
        \vspace{0.15em}
        \makebox[\linewidth][l]{%
            \hspace*{0.58\linewidth}%
        \makebox[0pt][c]{\small\textcolor{reflink}{\textbf{(b)}}}%
        }
    \end{minipage}
    \vspace{-0.5em}
    \caption{\textbf{Natural image geometry from video time steps.}
    \textcolor{reflink}{\textbf{(a)}} NLL (bits/dim) of intrinsic IGM and observation-space LD-RMM, across model rank.
    \textcolor{reflink}{\textbf{(b)}} Explained variance of CDC eigenvectors.
    }
    \label{fig:video_results}
    \vspace{-1.5em}
\end{wrapfigure}

We do have access to \textit{natural image perturbations} from video, where the jumps between time steps represent an intrinsic sample from an underlying process.
We train IGM with a low-rank UNet backbone on $64 \times 64$ videos from CelebV-HQ \citep{zhu2022celebvhq}, and compare with LD-RMM.
We test whether \textcolor{reflink}{\textbf{(i)}} IGM with natural perturbations recovers the intrinsic process geometry, \textcolor{reflink}{\textbf{(ii)}} how efficiently IGM and LD-RMM learn the tangent spaces of the data.
Full details are in App. \ref{app:videos}, and we visualise the learned tangent vectors in \cref{fig:celebv_tangents_0.0} and \cref{fig:celebv_tangents_0.1}.

\paragraph{\textcolor{reflink}{\textbf{(i)}}}
We can measure \textbf{intrinsic geometry recovery} by interpreting the learned CDC as a Gaussian model and evaluating the negative log-likelihood (NLL) of held-out test transition pairs (\cref{fig:video_results}\textcolor{reflink}{\textbf{(a)}}).
Since IGM is trained directly on the process, it recovers the true geometry substantially more accurately at each rank.

\paragraph{\textcolor{reflink}{\textbf{(ii)}}}
As in the synthetic experiments, we can compare different CDCs by their \textbf{tangent space recovery}.
For a rank-$r$ model, at each dimension $k \le r$, we can compute the \textit{explained variance} of the first $k$ eigenvectors.
For test transition pairs $(x_t, x_{t + \tau})$, we can measure the fraction of the variance of $x_{t + \tau} - x_t$ retained after projection into the first $k$ CDC principal components, and plot the profile over $k$ in (\cref{fig:video_results}\textcolor{reflink}{\textbf{(b)}}).
IGM recovers the true tangent spaces in far fewer dimensions than LD-RMM.
We visualise the learned tangent vectors in \cref{fig:video_tangents} (see also \cref{fig:celebv_tangents_0.0} and \cref{fig:celebv_tangents_0.1}).

\begin{figure}[h]
    \centering
    \includegraphics[width=1\linewidth]{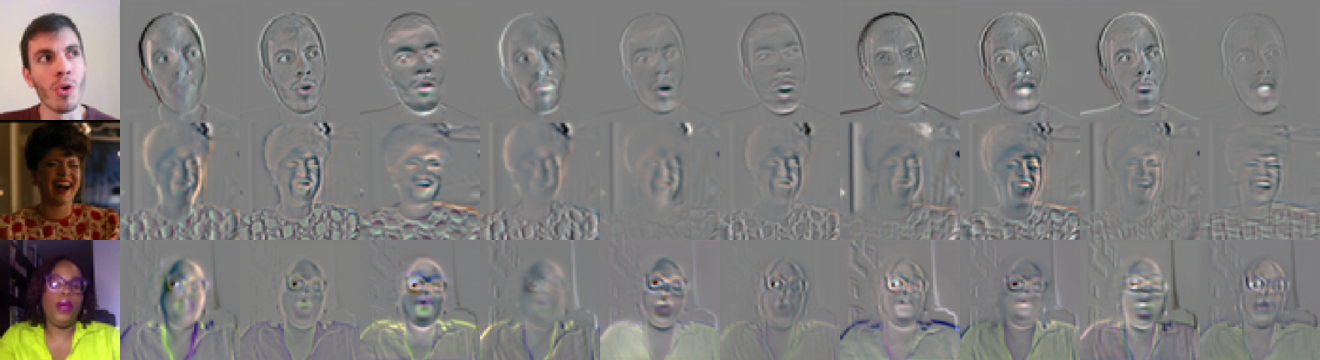}
    \caption{
    \textbf{Tangent vectors from video time steps.}
    Leading CDC eigenvectors reveal the tangent structure to the image `manifold', according to the \textit{intrinsic process} observed through videos.
    }
    \label{fig:video_tangents}
    \vspace{-0.5em}
\end{figure}

\subsection{Equipping latent representations with process geometry}
\label{subsec:molecular-dynamics}
Representation learning is usually focused on achieving a target latent distribution, but this is not always a good proxy for the geometry of the underlying process \citep{arvanitidis2017latent}.
IGM addresses this mismatch by learning the correct intrinsic geometry in the latent space. Using molecular dynamics in a frozen SLAE representation \citep{chen2025slae}, we test whether \textcolor{reflink}{\textbf{(i)}} IGM recovers the process geometry in latent coordinates, \textcolor{reflink}{\textbf{(ii)}} this geometry improves held-out transition paths, and \textcolor{reflink}{\textbf{(iii)}} observation equivariance transports this geometry to physically meaningful molecular motion.

\setlength{\columnsep}{12pt}
\begin{wrapfigure}{r}{0.42\textwidth}
    \vspace{-0.9em}
    \centering
    \includegraphics[width=0.48\linewidth]
        {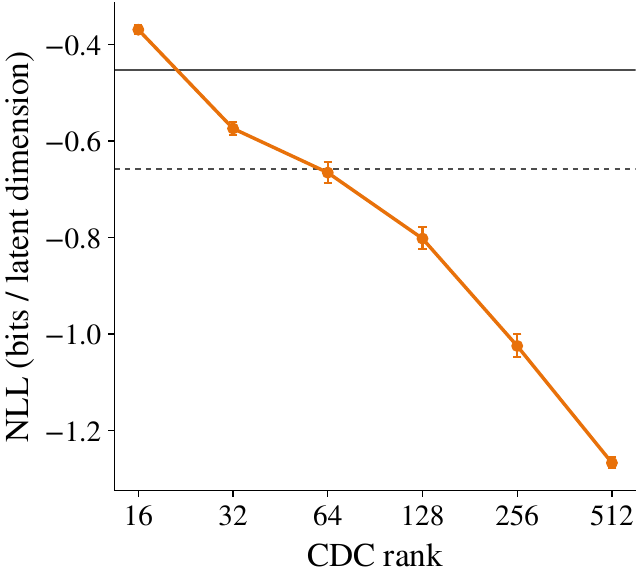}%
    \hfill
    \includegraphics[width=0.48\linewidth]
        {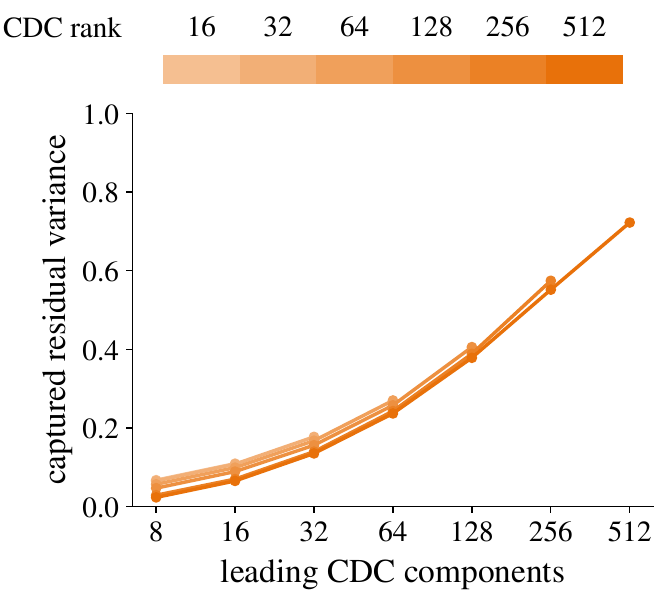}
    \caption{\textbf{Held-out transition fit.}
    NLL decreases with CDC rank (left), while leading
    directions capture increasing residual variability
    (right).}
    \label{fig:md-rank}
    \vspace{-2em}
\end{wrapfigure}

We use $200$ dynamic importance sampling trajectories of adenylate kinase \citep{perilla2011dims,beckstein2009adk} and a frozen SLAE encoder $\Phi$ \citep{chen2025slae} which maps each molecular structure $X_t$ to $214$ residue-wise embeddings with $128$ channels, so $Y_t=\Phi(X_t)\in\mathbb R^{214\times128}$.
Following \cref{subsec:diffspaces}, we apply IGM directly to the encoded transitions $(Y_t,Y_{t+\tau})$, learning the conditional mean $m_\tau^\phi$ and centred forward field $G_\tau^\theta$.
Full details are in App.~\ref{app:md}. \looseness=-1

\begin{figure}[ht]
\centering

\edef\mdPanelHeight{\the\dimexpr 0.25\linewidth\relax}
\edef\mdRowHeight{\the\dimexpr(\mdPanelHeight-4pt)/2\relax}

\newcommand{\mdthumb}[1]{%
    \includegraphics[
        width=0.19\linewidth,
        height=\mdRowHeight,
        keepaspectratio
    ]{#1}%
}

\captionsetup[subfigure]{
    labelformat=parens,
    labelfont={bf,color=reflink},
    labelsep=none,
    justification=centering,
    skip=4pt
}

\begin{subfigure}[b]{0.25\linewidth}
    \centering
    \begin{minipage}[c][\mdPanelHeight][c]{\linewidth}
        \centering
        \includegraphics[
            width=\linewidth,
            height=\mdPanelHeight,
            keepaspectratio
        ]{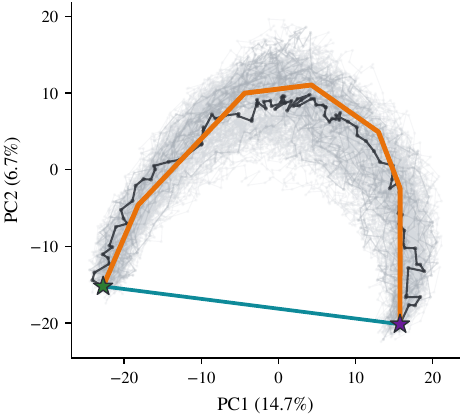}
    \end{minipage}
    \par
    \caption{}
    \label{fig:md-pca}
\end{subfigure}
\hfill
%
\begin{subfigure}[b]{0.72\linewidth}
    \centering
    \begin{minipage}[c][\mdPanelHeight][c]{\linewidth}
        \vbox to \mdPanelHeight{%
            \offinterlineskip

            \vbox to \mdRowHeight{%
                \vfil
                \hbox to \linewidth{%
                    \mdthumb{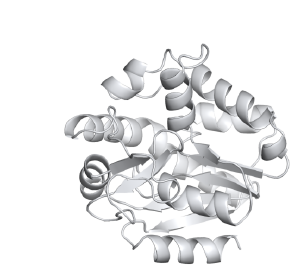}%
                    \hfil
                    \mdthumb{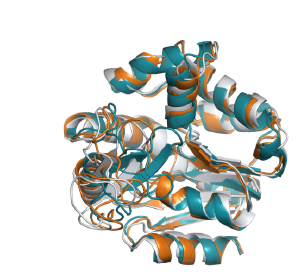}%
                    \hfil
                    \mdthumb{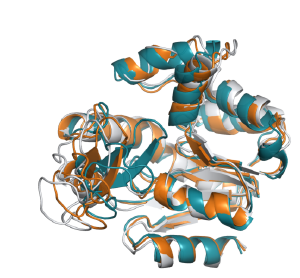}%
                    \hfil
                    \mdthumb{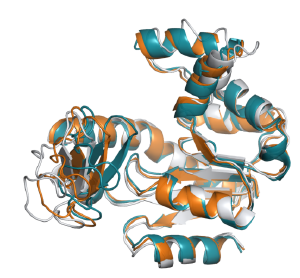}%
                    \hfil
                    \mdthumb{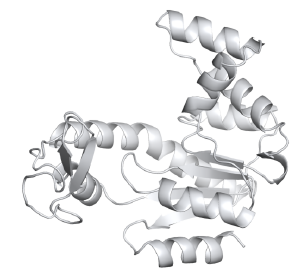}%
                }%
                \vfil
            }%

            \vskip 4pt

            \vbox to \mdRowHeight{%
                \vfil
                \hbox to \linewidth{%
                    \mdthumb{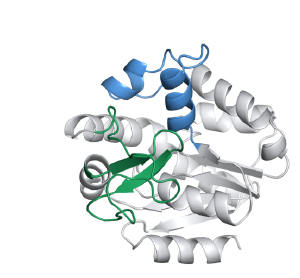}%
                    \hfil
                    \mdthumb{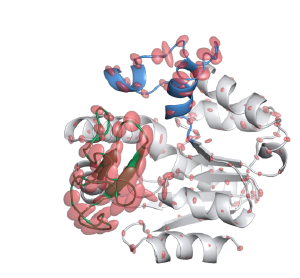}%
                    \hfil
                    \mdthumb{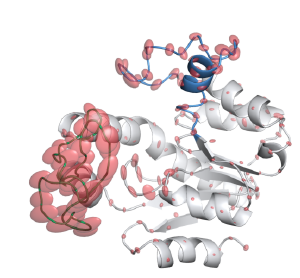}%
                    \hfil
                    \mdthumb{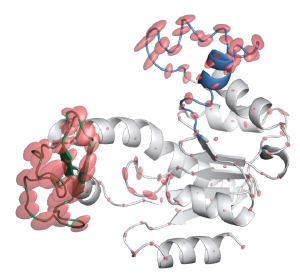}%
                    \hfil
                    \mdthumb{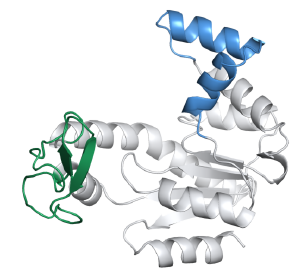}%
                }%
                \vfil
            }%
        }%
    \end{minipage}
    \par
    \caption{}
    \label{fig:md-structures}
\end{subfigure}

\caption{\textbf{Process geometry in molecular representations.}
\textcolor{reflink}{\textbf{(a)}} Held-out MD transition,
IGM bridge, and linear interpolation (LERP) between
fixed endpoints in latent PCA.
\textcolor{reflink}{\textbf{(b)}} Decoded paths (top)
and predicted per-residue mobility (bottom).
\textcolor{cdcorange}{\textbf{Orange}}: IGM;
\textcolor{lerpteal}{\textbf{teal}}: LERP;
\textcolor{mdheld}{\textbf{grey}}: held-out MD.
Results for all $20$ test trajectories appear in
\cref{fig:app-pca-panel}; additional decoded paths and
mobility ellipsoids appear in
\cref{fig:md-app-decoded-paths,fig:md-app-ellipsoids}.}
\label{fig:md}
\end{figure}

\paragraph{\textcolor{reflink}{\textbf{(i)}}}
As in \cref{subsec:videos}, we measure  \textbf{intrinsic geometry recovery} with NLL and explained variance of the CDC on test trajectories (\cref{fig:md}\textcolor{reflink}{\textbf{(a)}}).
NLL decreases consistently with rank and surpasses the best isotropic baseline, while increasing rank from $16$ to $512$ raises explained residual variance to $72.2\%$.
Together, these results reveal substantial anisotropic structure in the molecular transition process. \looseness=-1

\setlength{\columnsep}{2pt}
\begin{wrapfigure}{r}{0.2\linewidth}
    \vspace{-1em}
    \centering
    \scriptsize
    \setlength{\tabcolsep}{2pt}
    \renewcommand{\arraystretch}{0.9}
    \begin{tabular}{@{}lc@{}}
        \toprule
        \multicolumn{2}{@{}l}{\emph{Fr\'echet distance}} \\
        \midrule
        LERP & $2.00\pm0.14\,\text{\AA}$ \\
        \rowcolor{winrow}
        \textbf{IGM} & $\mathbf{1.40\pm0.16}\,\text{\AA}$ \\
        \bottomrule
    \end{tabular}
    \vspace{-1 em}
\end{wrapfigure}
\paragraph{\textcolor{reflink}{\textbf{(ii)}}}
We test whether the learned geometry improves \textbf{transition paths} between observed molecular states \citep{diepeveen2024riemannian}.
We construct an IGM bridge by conditioning the learned transition model on the observed endpoints, using its conditional mean and CDC along the predicted path (App.~\ref{app:md}).
Its conditional mean defines the IGM bridge.
\Cref{fig:md} shows one such transition in latent PCA \textcolor{reflink}{\textbf{(b)}} and as decoded structures \textcolor{reflink}{\textbf{(c}, top\textbf{)}}.
We compare each decoded path with the corresponding held-out MD trajectory using discrete Fr\'echet distance with C$\alpha$ Kabsch-RMSD as ground distance \citep{seyler2015psa}.
IGM improves on latent linear interpolation for all $20$ test trajectories, by $30.1\%$ on average.

\paragraph{\textcolor{reflink}{\textbf{(iii)}}}
We use observation equivariance to push forward the learned CDC through the frozen SLAE decoder $\Psi$ and visualise \textbf{molecular mobility} as per-residue displacement ellipsoids (\cref{fig:md}\textcolor{reflink}{\textbf{(c}, bottom\textbf{)}}).
The predicted mobility concentrates in the \textcolor{lidgreen}{\textbf{LID}} and \textcolor{nmpblue}{\textbf{NMP}} domains relative to the \textcolor{corescaf}{\textbf{CORE}}, consistent with the known open--closed motion of AdK \citep{beckstein2009adk}.
Thus, the learned latent process geometry can be transported back to reveal physically meaningful molecular motion.

\section{Conclusion}

We introduced \textbf{Intrinsic Geometry Matching (IGM)}: a scalable approach to learning the intrinsic geometry of data, thereby disentangling the underlying \textit{process geometry} from the distortion of the observation.
Our insight is that this geometry can be described with \textit{diffusion geometry} through its carré du champ (CDC) operator, and the CDC can be learned from intrinsic perturbations of the data, even when embedded non-isometrically.
This CDC is \textit{observation-equivariant}, transforming naturally between different measurements, and we identify the LogDet divergence as the unique observation-invariant loss for learning it.
We can exploit this equivariance to train a model in one space with data from another, potentially unlocking opportunities for scaling data in applications.
In each case, it outperformed the Gaussian (observation-dependent) and graph-based estimators.
We showed that IGM can induce geometry in latent spaces by learning the process geometry underlying molecular dynamics, recovering accurate trajectories and mobility fields.

\newpage

\subsection*{Acknowledgements}

AITHYRA is supported by the Austrian Academy of Sciences and the not-for-profit Boehringer Ingelheim Stiftung.
This work is partially supported by the EPSRC Turing AI World-Leading Research Fellowship
No. EP/X040062/1 and EPSRC AI Hub No. EP/Y028872/1.

\subsection*{AI use statement}

Generative AI was used for code development, debugging, and refining figures, and providing feedback on the final manuscript.
The authors take full responsibility for the final content of the paper.



\bibliography{refs}

@book{bakry2013analysis,
  title={Analysis and Geometry of Markov Diffusion Operators},
  author={Bakry, Dominique and Gentil, Ivan and Ledoux, Michel},
  volume={348},
  year={2013},
  publisher={Springer Science \& Business Media}
}

@misc{jones2024diffusiongeometry,
      title={Diffusion Geometry}, 
      author={Iolo Jones},
      year={2024},
      eprint={2405.10858},
      archivePrefix={arXiv},
      primaryClass={math.MG},
      url={https://arxiv.org/abs/2405.10858}, 
}

@misc{jones2026computingdiffusiongeometry,
      title={Computing Diffusion Geometry}, 
      author={Iolo Jones and David Lanners},
      year={2026},
      eprint={2602.06006},
      archivePrefix={arXiv},
      primaryClass={math.DG},
      url={https://arxiv.org/abs/2602.06006}, 
}

@article{jones2026manifold,
  title={Manifold diffusion geometry: Curvature, tangent spaces, and dimension},
  author={Jones, Iolo},
  journal={SIAM Journal on Mathematics of Data Science},
  volume={8},
  number={2},
  pages={257--277},
  year={2026},
  publisher={SIAM}
}

@article{coifman2006,
title = {Diffusion maps},
journal = {Applied and Computational Harmonic Analysis},
volume = {21},
number = {1},
pages = {5-30},
year = {2006},
note = {Special Issue: Diffusion Maps and Wavelets},
issn = {1063-5203},
doi = {https://doi.org/10.1016/j.acha.2006.04.006},
url = {https://www.sciencedirect.com/science/article/pii/S1063520306000546},
author = {Ronald R. Coifman and Stéphane Lafon}
}

@article{banerjee2005optimality,
  author  = {Banerjee, Arindam and Guo, Xin and Wang, Hui},
  title   = {On the Optimality of Conditional Expectation as a Bregman Predictor},
  journal = {IEEE Transactions on Information Theory},
  volume  = {51},
  number  = {7},
  pages   = {2664--2669},
  year    = {2005},
  doi     = {10.1109/TIT.2005.850145}
}

@article{kulis2009low,
  author  = {Kulis, Brian and Sustik, M{\'a}ty{\'a}s A. and Dhillon, Inderjit S.},
  title   = {Low-Rank Kernel Learning with Bregman Matrix Divergences},
  journal = {Journal of Machine Learning Research},
  volume  = {10},
  pages   = {341--376},
  year    = {2009}
}

@inproceedings{manor2024posterior,
  title     = {On the Posterior Distribution in Denoising: Application to Uncertainty Quantification},
  author    = {Manor, Hila and Michaeli, Tomer},
  booktitle = {The Twelfth International Conference on Learning Representations},
  year      = {2024}
}

@inproceedings{meng2021estimating,
  title     = {Estimating High Order Gradients of the Data Distribution by Denoising},
  author    = {Meng, Chenlin and Song, Yang and Li, Wenzhe and Ermon, Stefano},
  booktitle = {Advances in Neural Information Processing Systems},
  volume    = {34},
  pages     = {25359--25369},
  year      = {2021}
}

@incollection{james1961estimation,
  author    = {James, W. and Stein, Charles},
  title     = {Estimation with Quadratic Loss},
  booktitle = {Proceedings of the Fourth Berkeley Symposium on Mathematical Statistics and Probability, Volume 1: Contributions to the Theory of Statistics},
  publisher = {University of California Press},
  address   = {Berkeley, Calif.},
  year      = {1961},
  pages     = {361--379}
}

@article{seyler2015psa,
    doi = {10.1371/journal.pcbi.1004568},
    author = {Seyler, Sean L. AND Kumar, Avishek AND Thorpe, M. F. AND Beckstein, Oliver},
    journal = {PLOS Computational Biology},
    publisher = {Public Library of Science},
    title = {Path Similarity Analysis: A Method for Quantifying Macromolecular Pathways},
    year = {2015},
    month = {10},
    volume = {11},
    url = {https://doi.org/10.1371/journal.pcbi.1004568},
    pages = {1-37},
    number = {10},

}

@article{chen2025slae,
  title   = {SLAE: Strictly Local All-atom Environment for Protein Representation},
  author  = {Chen, Yilin and Lu, Tianyu and Zhao, Cizhang and Wayment-Steele, Hannah K. and Huang, Po-Ssu},
  journal = {bioRxiv},
  year    = {2025},
  month   = oct,
  doi     = {10.1101/2025.10.03.680398},
  url     = {https://www.biorxiv.org/content/10.1101/2025.10.03.680398v1},
  publisher = {Cold Spring Harbor Laboratory},
  note    = {Preprint}
}

@article{beckstein2009adk,
  author  = {Beckstein, Oliver and Denning, Elizabeth J. and
             Perilla, Juan R. and Woolf, Thomas B.},
  title   = {Zipping and Unzipping of Adenylate Kinase:
             Atomistic Insights into the Ensemble of Open--Closed Transitions},
  journal = {Journal of Molecular Biology},
  volume  = {394},
  number  = {1},
  pages   = {160--176},
  year    = {2009},
  doi     = {10.1016/j.jmb.2009.09.009}
}

@article{perilla2011dims,
  author  = {Perilla, Juan R. and Beckstein, Oliver and
             Denning, Elizabeth J. and Woolf, Thomas B.},
  title   = {Computing Ensembles of Transitions from Stable States:
             Dynamic Importance Sampling},
  journal = {Journal of Computational Chemistry},
  volume  = {32},
  number  = {2},
  pages   = {196--209},
  year    = {2011},
  doi     = {10.1002/jcc.21564}
}

@book{younes2010shapes,
  title     = {Shapes and Diffeomorphisms},
  author    = {Younes, Laurent},
  year      = {2010},
  publisher = {Springer Berlin Heidelberg},
  address   = {Berlin, Heidelberg},
  doi       = {10.1007/978-3-642-12055-8},
  isbn      = {978-3-642-12054-1}
}

@misc{franois2021metamorphicimageregistrationusing,
      title={Metamorphic image registration using a semi-Lagrangian scheme}, 
      author={Anton François and Pietro Gori and Joan Glaunès},
      year={2021},
      eprint={2106.08817},
      archivePrefix={arXiv},
      primaryClass={cs.CV},
      url={https://arxiv.org/abs/2106.08817}, 
}

@misc{holm2008eulerpoincaretheorymetamorphosis,
      title={The Euler-Poincare theory of Metamorphosis}, 
      author={Darryl D. Holm and Alain Trouve and Laurent Younes},
      year={2008},
      eprint={0806.0870},
      archivePrefix={arXiv},
      primaryClass={cs.CV},
      url={https://arxiv.org/abs/0806.0870}, 
}

@misc{kharitenko2026landingscoreriemannianoptimization,
      title={Landing with the Score: Riemannian Optimization through Denoising}, 
      author={Andrey Kharitenko and Zebang Shen and Riccardo de Santi and Niao He and Florian Doerfler},
      year={2026},
      eprint={2509.23357},
      archivePrefix={arXiv},
      primaryClass={cs.LG},
      url={https://arxiv.org/abs/2509.23357}, 
}

@misc{bamberger2026riemannianmetricmatchingscalable,
      title={Riemannian Metric Matching for Scalable Geometric Modeling of Distributions}, 
      author={Jacob Bamberger and Adam Gosztolai and Pierre Vandergheynst and Michael Bronstein and Iolo Jones},
      year={2026},
      eprint={2606.14334},
      archivePrefix={arXiv},
      primaryClass={cs.LG},
      url={https://arxiv.org/abs/2606.14334}, 
}

@article{belkin2003laplacian,
  title={Laplacian eigenmaps for dimensionality reduction and data representation},
  author={Belkin, Mikhail and Niyogi, Partha},
  journal={Neural computation},
  volume={15},
  number={6},
  pages={1373--1396},
  year={2003},
  publisher={MIT Press}
}

@inproceedings{bamberger2026carre,
  title={Carr{\'e} du champ flow matching: better quality-generalisation tradeoff in generative models},
  author={Bamberger, Jacob and Jones, Iolo and Duncan, Dennis and Bronstein, Michael and Vandergheynst, Pierre and Gosztolai, Adam},
  booktitle={International Conference on Learning Representations},
  volume={2026},
  pages={135754--135779},
  year={2026}
}

@incollection{bakry1985diffusions,
  title={Diffusions hypercontractives},
  author={Bakry, Dominique and {\'E}mery, Michel},
  booktitle={S{\'e}minaire de Probabilit{\'e}s XIX 1983/84: Proceedings},
  pages={177--206},
  year={1985},
  publisher={Springer}
}

@article{kunita1969absolute,
  title={Absolute continuity of Markov processes and generators},
  author={Kunita, Hiroshi},
  journal={Nagoya Mathematical Journal},
  volume={36},
  pages={1--26},
  year={1969},
  publisher={Cambridge University Press}
}

@article{tenenbaum2000global,
  title={A global geometric framework for nonlinear dimensionality reduction},
  author={Tenenbaum, Joshua B and Silva, Vin de and Langford, John C},
  journal={science},
  volume={290},
  number={5500},
  pages={2319--2323},
  year={2000},
  publisher={American Association for the Advancement of Science}
}

@article{roweis2000nonlinear,
  title={Nonlinear dimensionality reduction by locally linear embedding},
  author={Roweis, Sam T and Saul, Lawrence K},
  journal={science},
  volume={290},
  number={5500},
  pages={2323--2326},
  year={2000},
  publisher={American Association for the Advancement of Science}
}

@article{dawid1999coherent,
  title={Coherent dispersion criteria for optimal experimental design},
  author={Dawid, A Philip and Sebastiani, Paola},
  journal={Annals of Statistics},
  pages={65--81},
  year={1999},
  publisher={JSTOR}
}

@article{bengio2012representation,
  title={Representation learning: A review and new perspectives. arXiv 2012},
  author={Bengio, Yoshua and Courville, Aaron and Vincent, Pascal},
  journal={arXiv preprint arXiv:1206.5538},
  volume={10},
  year={2012}
}

@article{berry2020spectral,
  title={Spectral exterior calculus},
  author={Berry, Tyrus and Giannakis, Dimitrios},
  journal={Communications on Pure and Applied Mathematics},
  volume={73},
  number={4},
  pages={689--770},
  year={2020},
  publisher={Wiley Online Library}
}

@article{song2019generative,
  title={Generative modeling by estimating gradients of the data distribution},
  author={Song, Yang and Ermon, Stefano},
  journal={Advances in neural information processing systems},
  volume={32},
  year={2019}
}

@article{song2020score,
  title={Score-based generative modeling through stochastic differential equations},
  author={Song, Yang and Sohl-Dickstein, Jascha and Kingma, Diederik P and Kumar, Abhishek and Ermon, Stefano and Poole, Ben},
  journal={arXiv preprint arXiv:2011.13456},
  year={2020}
}

@article{kawasaki2026diffusion,
  title={Diffusion processes on implicit manifolds},
  author={Kawasaki-Borruat, Victor and Grotehans, Clara and Vandergheynst, Pierre and Gosztolai, Adam},
  journal={arXiv preprint arXiv:2604.07213},
  year={2026}
}

@inproceedings{zhu2022celebvhq,
  title={{CelebV-HQ}: A Large-Scale Video Facial Attributes Dataset},
  author={Zhu, Hao and Wu, Wayne and Zhu, Wentao and Jiang, Liming and Tang, Siwei and Zhang, Li and Liu, Ziwei and Loy, Chen Change},
  booktitle={ECCV},
  year={2022}
}

@article{singer2008non,
  title={Non-linear independent component analysis with diffusion maps},
  author={Singer, Amit and Coifman, Ronald R},
  journal={Applied and Computational Harmonic Analysis},
  volume={25},
  number={2},
  pages={226--239},
  year={2008},
  publisher={Elsevier}
}

@article{singer2009detecting,
  title={Detecting intrinsic slow variables in stochastic dynamical systems by anisotropic diffusion maps},
  author={Singer, Amit and Erban, Radek and Kevrekidis, Ioannis G and Coifman, Ronald R},
  journal={Proceedings of the National Academy of Sciences},
  volume={106},
  number={38},
  pages={16090--16095},
  year={2009},
  publisher={National Academy of Sciences}
}

@article{talmon2013empirical,
  title={Empirical intrinsic geometry for nonlinear modeling and time series filtering},
  author={Talmon, Ronen and Coifman, Ronald R},
  journal={Proceedings of the National Academy of Sciences},
  volume={110},
  number={31},
  pages={12535--12540},
  year={2013},
  publisher={National Academy of Sciences}
}

@inproceedings{davis2007information,
  title={Information-theoretic metric learning},
  author={Davis, Jason V and Kulis, Brian and Jain, Prateek and Sra, Suvrit and Dhillon, Inderjit S},
  booktitle={Proceedings of the 24th international conference on Machine learning},
  pages={209--216},
  year={2007}
}

@article{lebanon2012learning,
  title={Learning riemannian metrics},
  author={Lebanon, Guy},
  journal={arXiv preprint arXiv:1212.2474},
  year={2012}
}

@article{diepeveen2024riemannian,
  title={Riemannian geometry for efficient analysis of protein dynamics data},
  author={Diepeveen, Willem and Esteve-Yag{\"u}e, Carlos and Lellmann, Jan and {\"O}ktem, Ozan and Sch{\"o}nlieb, Carola-Bibiane},
  journal={Proceedings of the National Academy of Sciences},
  volume={121},
  number={33},
  pages={e2318951121},
  year={2024},
  publisher={National Academy of Sciences}
}

@article{hauberg2012geometric,
  title={A geometric take on metric learning},
  author={Hauberg, S{\o}ren and Freifeld, Oren and Black, Michael},
  journal={Advances in Neural Information Processing Systems},
  volume={25},
  year={2012}
}

@article{arvanitidis2017latent,
  title={Latent space oddity: on the curvature of deep generative models},
  author={Arvanitidis, Georgios and Hansen, Lars Kai and Hauberg, S{\o}ren},
  journal={arXiv preprint arXiv:1710.11379},
  year={2017}
}
\bibliographystyle{plainnat}

\newpage
\appendix

\section*{Appendix contents}

\begingroup
\small
\setlength{\parindent}{0pt}
\setlength{\parskip}{0.15em}

\hyperref[app:notation]{Notation}
\dotfill
\pageref{app:notation}

\hyperref[app:proofs]{A. Naturality of the LogDet divergence}
\dotfill
\pageref{app:proofs}

\hyperref[app:low-rank]{B. Low-rank loss derivation}
\dotfill
\pageref{app:low-rank}

\hyperref[app:identifiability]{C. Identifiability from partial observations}
\dotfill
\pageref{app:identifiability}

\vspace{0.15em}
\textbf{\hyperref[app:exp-details]{D. Experimental details}}
\dotfill
\pageref{app:exp-details}

\hspace{1.5em}\hyperref[app:synthetic]{D.1 Recovering process geometry}
\dotfill
\pageref{app:synthetic}

\hspace{1.5em}\hyperref[app:mnist]{D.2 Image geometry through augmentation}
\dotfill
\pageref{app:mnist}

\hspace{1.5em}\hyperref[app:videos]{D.3 Natural image geometry from videos}
\dotfill
\pageref{app:videos}

\hspace{1.5em}\hyperref[app:md]{D.4 Equipping latent representations with process geometry}
\dotfill
\pageref{app:md}

\endgroup

\section*{Notation}
\label{app:notation}

\begingroup
\small
\renewcommand{\arraystretch}{1.15}
\noindent
\begin{tabularx}{\linewidth}{@{}lX@{}}
\toprule
Symbol & Meaning \\
\midrule
$\mathcal X,\mu$
    & Latent state space and sampling measure. \\
$\alpha,\Phi,X$
    & Latent state, observation map, and observation
      $X=\Phi(\alpha)$. \\
$d,D,r$
    & Intrinsic dimension, observation dimension,
      and number of columns in the low-rank factor. \\
$\tau$
    & Transition lag or perturbation scale. \\
$\mathcal L,\Gamma$
    & Diffusion generator and carr\'e du champ operator. \\
$g,g^{-1}$
    & Riemannian metric and cometric. \\
$D_x\Psi$
    & Differential of $\Psi$ at $x$. \\
$\Phi^*,\Phi_\sharp$
    & Pullback of functions and pushforward of
      measures or the CDC, as appropriate. \\
$G_\Phi$
    & Coordinate matrix of the observed infinitesimal CDC. \\
$G_\tau$
    & Conditional second moment of the observed
      increment, divided by $2\tau$. \\
$G_{\tau,\mathrm{Cov}}$
    & Conditional covariance of the successor,
      divided by $2\tau$. \\
$m_\tau^\phi,G_\tau^\theta$
    & Learned conditional successor mean and
      learned finite-scale CDC matrix. \\
$M^\theta,\epsilon^\theta$
    & Low-rank factor and learned isotropic floor in
      $G^\theta=M^\theta(M^\theta)^\top+\epsilon^\theta I_D$. \\
$\lambda$
    & Fixed additive regularisation of target
      and model matrices. \\
$\mathcal D_\omega,\mathcal D_{\mathrm{LD}}$
    & Bregman divergence and LogDet divergence. \\
$\mathbb S^D,\mathbb S_+^D,\mathbb S_{++}^D$
    & Symmetric, positive-semidefinite, and
      positive-definite $D\times D$ matrices. \\
\bottomrule
\end{tabularx}
\endgroup

\section{Naturality of the LogDet divergence}
\label{app:proofs}

We prove that congruence invariance characterises the LogDet
divergence, up to a positive scale, within the class of Bregman
divergences considered in \Cref{prop:ld-naturality}.
The argument first identifies the local Hessian form of a Bregman
divergence and classifies congruence-invariant bilinear forms.
Requiring such a form to be a Hessian then determines the LogDet
generator up to affine terms. Naturality under nonlinear changes
of observation follows pointwise from observation equivariance.

\begin{lemma}[Local Bregman geometry]
\label{lem:bregman-local}
Let $\mathcal{D}_\omega$ be as in \Cref{def:bregman_div}, with
$\omega\in C^2(\mathcal A)$. For $B\in\operatorname{int}(\mathcal A)$,
\[
\mathcal{D}_\omega(B+\epsilon H,B)
=
\frac{\epsilon^2}{2}\,g^\omega_B(H,H)+o(\epsilon^2),
\qquad
g^\omega_B(H,K):=\nabla^2\omega(B)[H,K].
\]
Thus $\mathcal{D}_\omega$ locally induces the Hessian metric $g^\omega$.
\end{lemma}

\begin{proof}
Taylor expansion gives
\[
\omega(B+\epsilon H)
=
\omega(B)
+\epsilon\langle\nabla\omega(B),H\rangle_{\mathrm F}
+\frac{\epsilon^2}{2}\nabla^2\omega(B)[H,H]
+o(\epsilon^2).
\]
Substitution into \Cref{def:bregman_div} cancels the zeroth- and
first-order terms, yielding the claim.
\end{proof}

\begin{lemma}[Congruence-invariant bilinear forms]
\label{lem:congruence-metric}
Let $g_B$ be a symmetric bilinear form on $\mathbb S^D$, depending on
$B\in\mathbb S_{++}^D$, and satisfying
\[
g_{JBJ^\top}(JHJ^\top,JKJ^\top)=g_B(H,K)
\]
for all $B\succ0$, $H,K\in\mathbb S^D$, and $J\in\mathrm{GL}(D)$.
Then, for $D\geq2$,
\[
g_B(H,K)
=
a\,\operatorname{tr}(B^{-1}HB^{-1}K)
+
b\,\operatorname{tr}(B^{-1}H)\operatorname{tr}(B^{-1}K)
\]
for some $a,b\in\mathbb R$.
\end{lemma}

\begin{proof}
The congruence action of $\mathrm{GL}(D)$ on $\mathbb S_{++}^D$ is
transitive, since $B^{-1/2}BB^{-1/2}=I$ for every $B\succ0$.
Hence $g$ is determined by its value at $I$. Since the stabilizer of
$I$ is $O(D)$,
\[
g_I(QHQ^\top,QKQ^\top)=g_I(H,K),
\qquad Q\in O(D).
\]

Under orthogonal conjugation,
\[
\mathbb S^D
=
\mathbb RI\oplus\mathbb S_0^D,
\qquad
\mathbb S_0^D
:=
\{H\in\mathbb S^D:\operatorname{tr}(H)=0\},
\]
and the action on $\mathbb S_0^D$ is irreducible. Therefore every
$O(D)$-invariant symmetric bilinear form has the form
\[
g_I(H,K)
=
a\,\operatorname{tr}(HK)
+
b\,\operatorname{tr}(H)\operatorname{tr}(K).
\]

Applying congruence invariance with $J=B^{-1/2}$ gives
\[
g_B(H,K)
=
g_I\!\left(
B^{-1/2}HB^{-1/2},
B^{-1/2}KB^{-1/2}
\right).
\]
Substituting the expression for $g_I$ and using cyclicity of the trace
yields
\[
g_B(H,K)
=
a\,\operatorname{tr}(B^{-1}HB^{-1}K)
+
b\,\operatorname{tr}(B^{-1}H)\operatorname{tr}(B^{-1}K),
\]
as claimed.
\end{proof}

\logdetnaturality*
\begin{proof}
Assume first that $\mathcal{D}_\omega$ is congruence invariant. By
\Cref{lem:bregman-local}, its Hessian form satisfies
\[
g^\omega_{JBJ^\top}(JHJ^\top,JHJ^\top)
=
g^\omega_B(H,H).
\]
By the polarization identity
$g_B^\omega(H,K)
=\tfrac12\bigl[g_B^\omega(H+K,H+K)-g_B^\omega(H,H)-g_B^\omega(K,K)\bigr]$,
invariance of the quadratic form implies invariance of the full
bilinear form. Thus $g^\omega$ is congruence invariant, and
\Cref{lem:congruence-metric} gives, for $D\geq2$,
\begin{equation}
g^\omega_B(H,K)
=
a\,\operatorname{tr}(B^{-1}HB^{-1}K)
+
b\,\operatorname{tr}(B^{-1}H)\operatorname{tr}(B^{-1}K).
\label{eq:invariant-hessian}
\end{equation}

It remains to impose that $g^\omega$ is a Hessian metric. Since
$g^\omega=\nabla^2\omega$, we have
$Dg^\omega_B[L](H,K)=\nabla^3\omega(B)[L,H,K]$, which is symmetric in
$L,H,K$. Using $D(B^{-1})[L]=-B^{-1}LB^{-1}$, differentiation of
\eqref{eq:invariant-hessian} at $B=I$ in direction $L$ gives
\[
\begin{aligned}
Dg^\omega_I[L](H,K)
={}&
-a\bigl[
\operatorname{tr}(LHK)+\operatorname{tr}(HLK)
\bigr]
\\
&-b\bigl[
\operatorname{tr}(LH)\operatorname{tr}(K)
+\operatorname{tr}(H)\operatorname{tr}(LK)
\bigr].
\end{aligned}
\]
Exchanging $L$ and $H$ and using the symmetry of the third derivative
therefore yields
\[
b\Bigl[
\operatorname{tr}(H)\operatorname{tr}(LK)
-
\operatorname{tr}(L)\operatorname{tr}(HK)
\Bigr]
=0.
\]
Taking $H=I$ and $K=L$, with $L\neq0$ traceless, gives
$b\,D\,\operatorname{tr}(L^2)=0$, hence $b=0$. Therefore
\[
\nabla^2\omega(B)[H,K]
=
a\,\operatorname{tr}(B^{-1}HB^{-1}K)
=
a\,\nabla^2(-\log\det B)[H,K].
\]
Thus $\omega(B)+a\log\det B$ has vanishing Hessian and is affine:
\[
\omega(B)
=
-a\log\det B+\langle C,B\rangle+c_0.
\]
Affine terms do not affect Bregman divergences, while strict convexity
implies $a>0$. Consequently,
\[
\mathcal{D}_\omega(A,B)
=
a\Bigl[
\operatorname{tr}(B^{-1}A)
-\log\det(B^{-1}A)-D
\Bigr].
\]

For $D=1$, congruence invariance reduces to scale invariance. Applied
to the local Hessian form, this gives $\omega''(x)=a/x^2$, and the same
conclusion follows.

Conversely,
$(JBJ^\top)^{-1}(JAJ^\top)
=J^{-\top}(B^{-1}A)J^\top$ is similar to $B^{-1}A$.
Since trace and determinant are invariant under similarity,
\[
\mathcal{D}_{\mathrm{LD}}(JAJ^\top,JBJ^\top)
=
\mathcal{D}_{\mathrm{LD}}(A,B),
\]
which completes the proof.
\end{proof}

\nonlinldnat*
\begin{proof}
By observation equivariance (\eqref{eq:conjugation_pushfoward_cdc}),
\[
G_{\Psi\circ\Phi}(\Psi(x)) = (D_x\Psi)G_\Phi(x)(D_x\Psi)^\top.
\]
Since $\Psi$ is a local diffeomorphism at $x$, $D_x\Psi\in\mathrm{GL}(D)$.
The result then follows directly from the congruence invariance of $\mathcal{D}_{\mathrm{LD}}$ in \Cref{prop:ld-naturality}.
\end{proof}   

\section{Low-rank loss derivation}
\label{app:low-rank}

We derive the low-rank form of the regularised LogDet loss in \Cref{subsec:igm-cond-training}.
For $X,Y\in\mathbb R^D$ and fixed $\tau>0$, define
\[
e=Y-m_\tau^\phi(X), \qquad A_\phi=\frac{ee^\top}{2\tau}.
\]
The original CBMM target is recovered by setting $e=Y-X$.
When the conditional mean is pretrained and frozen, as in the molecular-dynamics experiments, $A_\phi$ is fixed during covariance optimisation.

We parametrise $G_\tau^\theta=M^\theta(M^\theta)^\top+\epsilon^\theta I_D$, where $M^\theta\in\mathbb R^{D\times r}$ and $\epsilon^\theta\geq 0$, suppressing dependence on $X$ and $\tau$.
For fixed regularisation $\lambda>0$, we use
\[
\mathcal{D}_{\mathrm{LD}}^\lambda(A_\phi,G_\tau^\theta)
=
\mathcal{D}_{\mathrm{LD}}
\bigl(A_\phi+\lambda I_D,G_\tau^\theta+\lambda I_D\bigr).
\]
Write
\[
\begin{gathered}
\eta=\epsilon^\theta+\lambda,
\qquad
B=G_\tau^\theta+\lambda I_D
=M^\theta(M^\theta)^\top+\eta I_D,\\
K_\eta=I_r+\eta^{-1}(M^\theta)^\top M^\theta,
\qquad
z=(M^\theta)^\top e.
\end{gathered}
\]
The regularised covariance score
\begin{equation}
\label{eq:low-rank-logdet}
\mathcal L_{\mathrm{score}}
=
\frac{e^\top B^{-1}e}{2\tau}
+\lambda\operatorname{tr}(B^{-1})
+\log\det B
\end{equation}
differs from the regularised LogDet divergence by the
target-only terms:
\[
\mathcal{D}_{\mathrm{LD}}^\lambda(A_\phi,G_\tau^\theta)
=
\mathcal L_{\mathrm{score}}
-\log\det(A_\phi+\lambda I_D)-D.
\]
These terms can be discarded when the target is fixed,
but not generally when the mean is learned.

The matrix determinant lemma and Woodbury identity give
\[
\log\det B=D\log\eta+\log\det K_\eta,
\qquad
B^{-1}
=
\eta^{-1}I_D
-\eta^{-2}M^\theta K_\eta^{-1}(M^\theta)^\top.
\]
In particular,
\[
\operatorname{tr}(B^{-1})
=
\eta^{-1}\left(D-r+\operatorname{tr}(K_\eta^{-1})\right).
\]
Substitution into \eqref{eq:low-rank-logdet} yields
\begin{equation}
\label{eq:low-rank-full-score}
\begin{aligned}
\mathcal L_{\mathrm{score}}
={}&
D\log\eta+\log\det K_\eta
+\frac{\|e\|^2}{2\tau\eta}
-\frac{z^\top K_\eta^{-1}z}{2\tau\eta^2}\\
&+
\frac{\lambda}{\eta}
\left(D-r+\operatorname{tr}(K_\eta^{-1})\right).
\end{aligned}
\end{equation}
Applying the determinant lemma to the target gives the
full regularised LogDet divergence:
\begin{equation}
\label{eq:low-rank-full-objective}
\begin{aligned}
\mathcal{D}_{\mathrm{LD}}^\lambda(A_\phi,G_\tau^\theta)
={}&
\mathcal L_{\mathrm{score}}-D\log\lambda\\
&-\log\!\left(1+\frac{\|e\|^2}{2\tau\lambda}\right)-D.
\end{aligned}
\end{equation}
Both expressions retain all parameter-dependent terms.

For a fixed target and $\epsilon^\theta=0$, we have
$\eta=\lambda$.
Writing $K=I_r+\lambda^{-1}(M^\theta)^\top M^\theta$,
both objectives reduce, up to parameter-independent terms,
to
\begin{equation}
\label{eq:low-rank-objective}
\log\det K+\operatorname{tr}(K^{-1})
-\frac{z^\top K^{-1}z}{2\tau\lambda^2}.
\end{equation}
This applies to $e=Y-X$ and to
$e=Y-m_\tau^\phi(X)$ when the mean is held fixed.
For a learned mean with fixed $\eta=\lambda$, the
regularised covariance score requires adding
$\|e\|^2/(2\tau\lambda)$ to
\eqref{eq:low-rank-objective}.
The full regularised LogDet divergence requires this
residual-norm term and the additional term
$-\log(1+\|e\|^2/(2\tau\lambda))$,
up to parameter-independent constants.
If the floor is learned, all $\eta$-dependent terms in
\eqref{eq:low-rank-full-score} must also be retained.
Thus, the regularised covariance score and divergence
are equivalent up to parameter-independent terms for a
fixed target, but not generally for joint
mean--covariance training.

All matrix inverses and determinants needed to evaluate
these objectives involve only the $r\times r$ matrix
$K_\eta$, while the target determinant reduces to a
scalar expression.
The remaining computations use $\|e\|^2$ and
$z=(M^\theta)^\top e$.
Using a Cholesky factorisation and triangular solves gives
total cost $O(Dr^2+r^3)$ and memory $O(Dr+r^2)$ per sample,
excluding network evaluation, without materialising a
$D\times D$ cometric.

\section{Identifiability from partial observations}
\label{app:identifiability}

We characterise when multiple partial observations, as described in \Cref{subsec:diffspaces}, uniquely determine a cometric from its differential pushforwards.
For differentiable observation maps
$\Psi_o:\mathbb R^D\to\mathbb R^{D_o}$, $o=1,\ldots,O$, and a fixed $x\in\mathbb R^D$, define
\[
\Lambda_x(A)
:=
\bigl((D_x\Psi_o)A(D_x\Psi_o)^\top\bigr)_{o=1}^O,
\qquad A\in\mathbb S^D.
\]
For a subspace $V\subseteq\mathbb R^D$, let
\[
\operatorname{Sym}^2(V)
:=
\operatorname{span}\{uv^\top+vu^\top:u,v\in V\}
\subseteq\mathbb S^D.
\]

\begin{lemma}[Identifiability from partial observations]
\label{lem:partial-identifiability}
The differential observations $\Lambda_x(G_\tau(x))$
uniquely determine every $G_\tau(x)\in\mathbb S_+^D$
if and only if
\[
\sum_{o=1}^O
\operatorname{Sym}^2\!\left(
\operatorname{range}\bigl((D_x\Psi_o)^\top\bigr)
\right)
=
\mathbb S^D,
\]
equivalently, $\ker\Lambda_x=\{0\}$.
\end{lemma}

\begin{proof}
Write $J_o=D_x\Psi_o$.
Under the Frobenius inner product, the adjoint of
$A\mapsto J_oAJ_o^\top$ is $H_o\mapsto J_o^\top H_oJ_o$,
whose image is
$\operatorname{Sym}^2(\operatorname{range}(J_o^\top))$.
Hence
\[
(\ker\Lambda_x)^\perp
=
\operatorname{im}\Lambda_x^*
=
\sum_{o=1}^O
\operatorname{Sym}^2(\operatorname{range}(J_o^\top)),
\]
which proves the criterion for injectivity on
$\mathbb S^D$.
If $0\neq A\in\ker\Lambda_x$, then, for sufficiently
small $\epsilon>0$, the distinct matrices
$I_D+\epsilon A$ and $I_D-\epsilon A$ are positive
definite and have identical observations.
Thus injectivity on $\mathbb S^D$ is also necessary
for identifiability on $\mathbb S_+^D$; sufficiency
is immediate.
\end{proof}

The condition requires the views to resolve every
quadratic interaction, not merely every direction:
$\Lambda_x$ must have rank $D(D+1)/2$.
For unrestricted ambient cometrics and coordinate
projections $\Psi_o(x)=(x_i)_{i\in S_o}$, this means
that every coordinate pair $\{i,j\}$, including $i=j$,
must occur together in at least one view.
We apply this criterion to the partial observations
in the synthetic experiments in \Cref{subsec:synthetic}.

\section{Experimental details}
\label{app:exp-details}

\subsection{Recovering process geometry}
\label{app:synthetic}

\paragraph{Latent process and observation map.}
Let $d=2m$ and
$\alpha=(u_1,v_1,\ldots,u_m,v_m)\in\mathbb T^d$,
with angular coordinates represented in $[0,2\pi)$.
We sample independent latent anchors
$\alpha_0\sim\operatorname{Unif}(\mathbb T^d)$
and Brownian transitions
\[
\alpha_\tau
=
\bigl(\alpha_0+\sqrt{2\tau}\,\xi\bigr)\pmod{2\pi},
\qquad \xi\sim\mathcal N(0,I_d),
\]
where $\xi$ is independent of $\alpha_0$.
This process has generator
$\Delta=\sum_{i=1}^d\partial_{\alpha_i}^2$
and cometric $I_d$.

We observe each two-dimensional factor through the
torus-of-revolution embedding
\begin{equation}
\label{eq:torus-block-embedding}
\Phi_0(u,v)
=
\begin{pmatrix}
(R+r\cos v)\cos u\\
(R+r\cos v)\sin u\\
r\sin v
\end{pmatrix},
\qquad R=2,\quad r=1.
\end{equation}
Concatenating the $m$ blocks gives
$\Phi=\Phi_0^{\times m}:\mathbb T^d\to\mathbb R^H$,
where $H=3d/2$.
For each block,
\[
(D_{(u,v)}\Phi_0)^\top D_{(u,v)}\Phi_0
=
\operatorname{diag}\!\left((R+r\cos v)^2,r^2\right),
\]
so the observation-induced metric differs from the
flat process metric.

To vary the observation dimension, we use the isometric
lift $\Psi_D(\alpha)=Q_D\Phi(\alpha)$ for $D\geq H$,
where $Q_D\in\mathbb R^{D\times H}$ satisfies
$Q_D^\top Q_D=I_H$.
For $D=H$, we set $Q_D=I_H$; for $D>H$, we obtain
$Q_D$ by QR factorisation of a fixed Gaussian matrix.
The observation-induced metric is preserved:
\[
(D_\alpha\Psi_D)^\top D_\alpha\Psi_D
=
(D_\alpha\Phi)^\top D_\alpha\Phi.
\]
\paragraph{Neural estimators.}
All neural CDC estimators use the low-rank parametrisation
of App.~\ref{app:low-rank} with rank $r=64$.
The backbone is a FiLM-conditioned residual MLP with
width $512$, four residual blocks, and a $128$-dimensional
sinusoidal embedding of the scale $\tau$.
All models are trained with Adam using a learning rate
of $2\times10^{-4}$ and batch size $1024$.

\paragraph{Evaluation and uncertainty.}
Within each setting, all methods are evaluated on the
same fixed set of $8192$ held-out latent anchors
$\alpha\sim\operatorname{Unif}(\mathbb T^d)$,
excluded from training and model selection.
The true tangent projector in observation space is
\[
P(\alpha)
=
(D_\alpha\Psi_D)
\bigl((D_\alpha\Psi_D)^\top D_\alpha\Psi_D\bigr)^{-1}
(D_\alpha\Psi_D)^\top.
\]
For each method, let $\widehat P(\alpha)$ be the
orthogonal projector onto the span of the leading $d$
eigenvectors of its estimated ambient CDC at
$\Psi_D(\alpha)$.
We measure tangent recovery by
\begin{equation}
\label{eq:tangent-projector-error}
\mathcal E_P
=
\mathbb E_\alpha\left[
\frac{
\|\widehat P(\alpha)-P(\alpha)\|_{\mathrm F}^2
}{2d}
\right],
\end{equation}
which is the average over anchors of the mean squared
sine of the principal angles between the estimated
and true tangent spaces.

For IGM and FGM, we additionally recover the learned
cometric in latent coordinates:
\[
G^{\theta,\mathrm{lat}}(\alpha)
=
(D_\alpha\Psi_D)^\dagger
G^\theta(\Psi_D(\alpha))
\bigl((D_\alpha\Psi_D)^\dagger\bigr)^\top,
\]
where $\dagger$ denotes the Moore--Penrose pseudoinverse
and the evaluation scale is suppressed in the notation.
Since the true latent cometric is $I_d$, we report
the relative latent-cometric error
\begin{equation}
\label{eq:latent-cometric-error}
\mathcal E_\Gamma
=
\mathbb E_\alpha\left[
\frac{
\|G^{\theta,\mathrm{lat}}(\alpha)-I_d\|_{\mathrm F}
}{\sqrt d}
\right].
\end{equation}
We estimate both expectations by averaging over the
held-out anchors and report averages across training
seeds.
For the sample-size experiments, curves show three-seed
means and bands span the seed range.
 
\paragraph{Geodesic optimisation.}
We compute geodesics in latent coordinates using the
metric $g^\theta$ obtained from the learned cometric.
For fixed endpoints $\alpha_0,\alpha_K\in\mathbb T^2$,
we optimise $K=64$ segments by minimising the midpoint
Dirichlet energy
\begin{equation}
\label{eq:geodesic-energy}
E(\alpha_0,\ldots,\alpha_K)
=
K\sum_{k=0}^{K-1}
\Delta\alpha_k^\top
g^\theta\!\left(\alpha_k+\tfrac12\Delta\alpha_k\right)
\Delta\alpha_k,
\qquad
\Delta\alpha_k
=
\operatorname{wrap}(\alpha_{k+1}-\alpha_k)
\in[-\pi,\pi)^2.
\end{equation}
Interior vertices are initialised along the shortest
wrapped linear path and optimised for $1200$ Adam
iterations with learning rate $0.03$, cosine-decayed
to zero.
Both estimators use the same endpoint pairs:
$((0.9,0.3),(3.3,1.6))$ and
$((3.7,1.4),(6.1,2.1))$.

\paragraph{Rank-deficient views.}
For \textcolor{reflink}{\textbf{(iv)}}, we set $d=2$
and replace the isotropic Brownian motion above with
an anisotropic diffusion on the flat torus:
\[
\alpha_\tau
=
\bigl(\alpha_0+\sqrt{2\tau}\,A^{1/2}\xi\bigr)
\pmod{2\pi},
\qquad
\xi\sim\mathcal N(0,I_2),
\qquad
A=
\begin{pmatrix}
1&0.6\\
0.6&1
\end{pmatrix}.
\]
This process has cometric $A$, whose off-diagonal
entries describe correlated fluctuations of the two
angular coordinates.
We observe it through winding characters
\[
\psi_{\mathbf m}(\alpha)
=
\bigl(\cos\langle\mathbf m,\alpha\rangle,
\sin\langle\mathbf m,\alpha\rangle\bigr),
\qquad
\mathbf m\in\mathbb Z^2\setminus\{0\}.
\]
Each map records the angle $\langle\mathbf m,\alpha\rangle$
modulo $2\pi$, so it is globally non-injective and its
differential has rank one.
Writing
$t_{\mathbf m}(\alpha)
=(-\sin\langle\mathbf m,\alpha\rangle,
\cos\langle\mathbf m,\alpha\rangle)^\top$,
we have
$D_\alpha\psi_{\mathbf m}
=t_{\mathbf m}(\alpha)\mathbf m^\top$ and
\[
(D_\alpha\psi_{\mathbf m})A(D_\alpha\psi_{\mathbf m})^\top
=
(\mathbf m^\top A\mathbf m)\,
t_{\mathbf m}(\alpha)t_{\mathbf m}(\alpha)^\top.
\]
Each differential observation therefore reveals only
the quadratic measurement $\mathbf m^\top A\mathbf m$.
We use the nested sets
\[
\begin{aligned}
\mathcal M_1&=\{(1,0)\},&
\mathcal M_2&=\mathcal M_1\cup\{(0,1)\},\\
\mathcal M_3&=\mathcal M_2\cup\{(1,1)\},&
\mathcal M_4&=\mathcal M_3\cup\{(1,-1)\}.
\end{aligned}
\]
By \Cref{lem:partial-identifiability}, these differential
observations identify an unrestricted latent cometric
precisely when
$\{\mathbf m\mathbf m^\top:\mathbf m\in\mathcal M_O\}$
spans $\mathbb S^2$.
The first two views reveal the diagonal entries but
leave the cross term unresolved; the third completes
the span, giving the threshold $O^\star=3$.

For each $O\in\{1,2,3,4\}$, we train IGM on $N=10^5$
anchors for $200$ epochs with four seeds.
We use rank $r=8$, width $512$, four residual blocks,
batch size $1024$, and
$\tau\sim\operatorname{Unif}[0.005,0.1]$,
with the debiased LogDet loss and $\lambda=0.1$.
We evaluate at $\tau=0.005$ on $4000$ held-out anchors
and report the latent-cometric error against $A$.

\subsection{Image geometry through augmentation}
\label{app:mnist}

\setlength{\columnsep}{18pt}
\begin{wrapfigure}{r}{0.33\linewidth}
    \vspace{-1em}
    \centering
    \includegraphics[width=\linewidth]{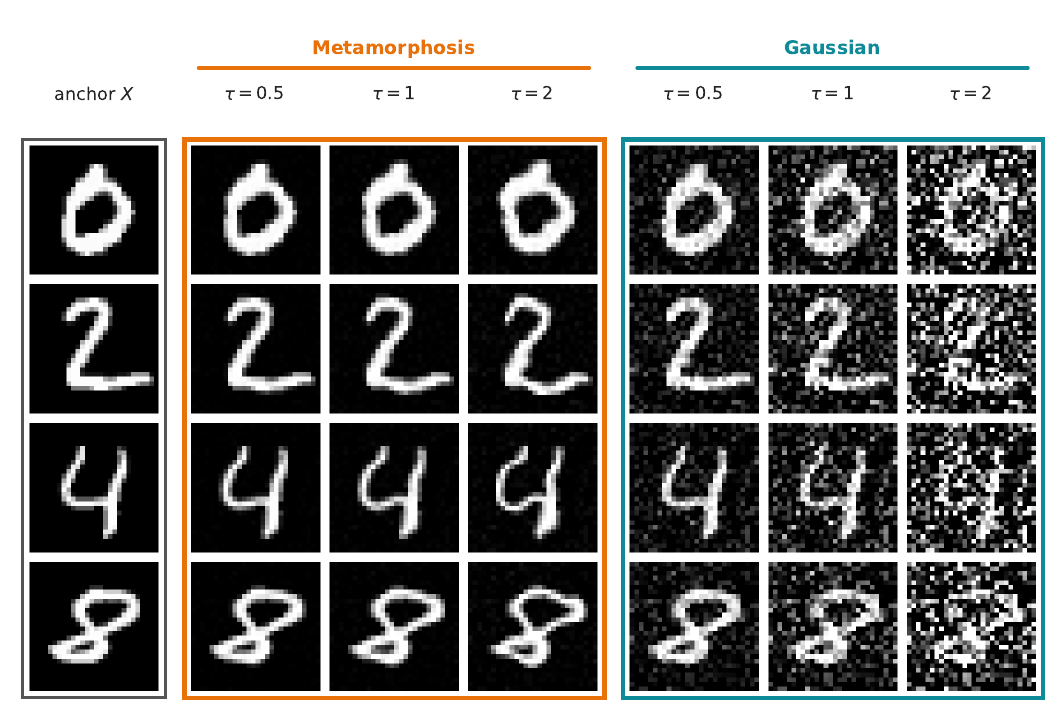}
    \caption{\textbf{Sampled image transitions.} Samples at increasing transition scales $\tau\in\{0.5,1,2\}$. \textcolor{transitioncolor}{\textbf{Orange}}: metamorphosis; \textcolor{gaussiancolor}{\textbf{teal}}: isotropic Gaussian perturbations.}
\label{fig:kernel_samples}
\end{wrapfigure}
\paragraph{Data and transition process.}
We train on the official MNIST training split using
$28\times28$ greyscale images.
For IGM, given an anchor $X$, we sample image momentum
$z=0.3\varepsilon$, where $\varepsilon\sim\mathcal N(0,I)$,
and generate a metamorphic transition by 
$v=-\eta K_\sigma\star(z\nabla X)$, with
$\phi_\tau=\exp(-\tau v)$, and  $Y=X\circ\phi_\tau +\frac{\tau\kappa}{2}(z+z\circ\phi_\tau)$.
Here, $K_\sigma$ is a Gaussian smoothing kernel,
$\star$ denotes spatial convolution, and $\phi_\tau$
is the time-$\tau$ flow of the stationary velocity
field $-v$.
The transition combines spatial deformation with
an additive intensity change.
We use $\eta=1.2$, $\kappa=0.05$, and $\sigma=2$ pixels
with $13\times13$ kernel support, sampling
$\tau\sim\operatorname{LogUnif}[0.2,0.4]$
independently per example.
Spatial derivatives use centred differences, and
generated transitions are not clipped.
Examples are shown in \Cref{fig:kernel_samples}.

\paragraph{Field estimation and training.}
IGM learns the uncentred forward field from metamorphic
transitions, while LD-RMM learns a mean-centred backward
field from isotropic Gaussian perturbations.
Both CDC estimators use a scale-conditioned U-Net
with base width $64$, channel multipliers $(1,2,2)$,
two residual blocks per resolution, and no dropout.
Single-head self-attention is applied at downsampling
factor $4$ and in the middle block.
A sinusoidal embedding of $\tau$, followed by a
two-layer MLP, conditions each residual block.
The CDC head outputs $100$ spatial channels, flattened
into $U_\theta\in\mathbb R^{100\times784}$,
and the cometric includes a learned isotropic floor. We use Adam with learning rate $2\times10^{-4}$,
batch size $128$, no weight decay, and gradient
clipping at norm $1$.
After the fixed validation objective plateaus for
$15$ epochs, we trigger a $100$-epoch cosine decay
of the learning rate to zero, with training capped
at $200$ epochs.

\paragraph{Target-directed trajectories.}
As described in the main text, we compute interpolating
paths by Riemannian optimisation, evaluating the learned
field at $\tau=0.3$.
For each source--target pair $(x_0,x_1)$, we take
normalised Riemannian-gradient steps on the target
potential $\frac12\|x-x_1\|_2^2$, using the learned
low-rank cometric:
\begin{equation}
\label{eq:mnist-normalized-natural-gradient}
g_k=x_k-x_1,
\qquad
d_k=
\frac{U_k^\top U_k g_k}{\|U_k g_k\|_2},
\qquad
x_{k+1}=x_k-\alpha d_k.
\end{equation}
Here, $U_k$ is evaluated at $x_k$, and the denominator
normalises the Riemannian-gradient direction.
We use $8192$ updates and calibrate the step size
$\alpha$ once per field using the same source--target
pairs.
The field is re-evaluated at every iterate, with no
denoising, projection, or post-step clipping.
We apply the MNIST-trained models to EMNIST letters
without fine-tuning.
\Cref{fig:emnist-scale} shows eigenvectors at additional
scales, and \Cref{fig:additional-image-interpolations}
shows additional paths.

\begin{figure}[h]
\centering

\begin{subfigure}[c]{0.32\linewidth}
    \centering
    \includegraphics[width=\linewidth]{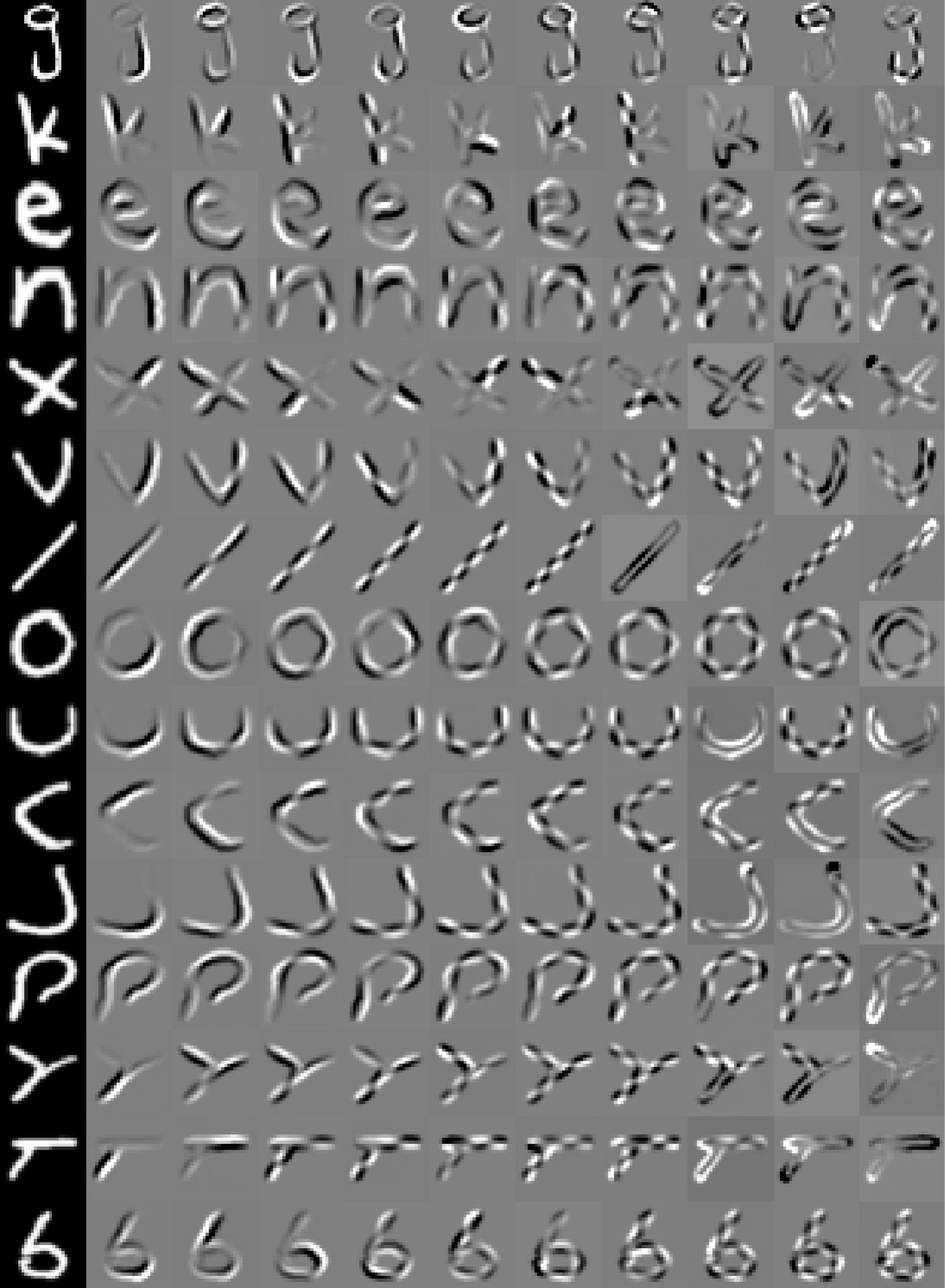}
    \caption*{\textcolor{reflink}{\textbf{(a)}}}
\end{subfigure}
\hfill
\begin{subfigure}[c]{0.32\linewidth}
    \centering
    \includegraphics[width=\linewidth]{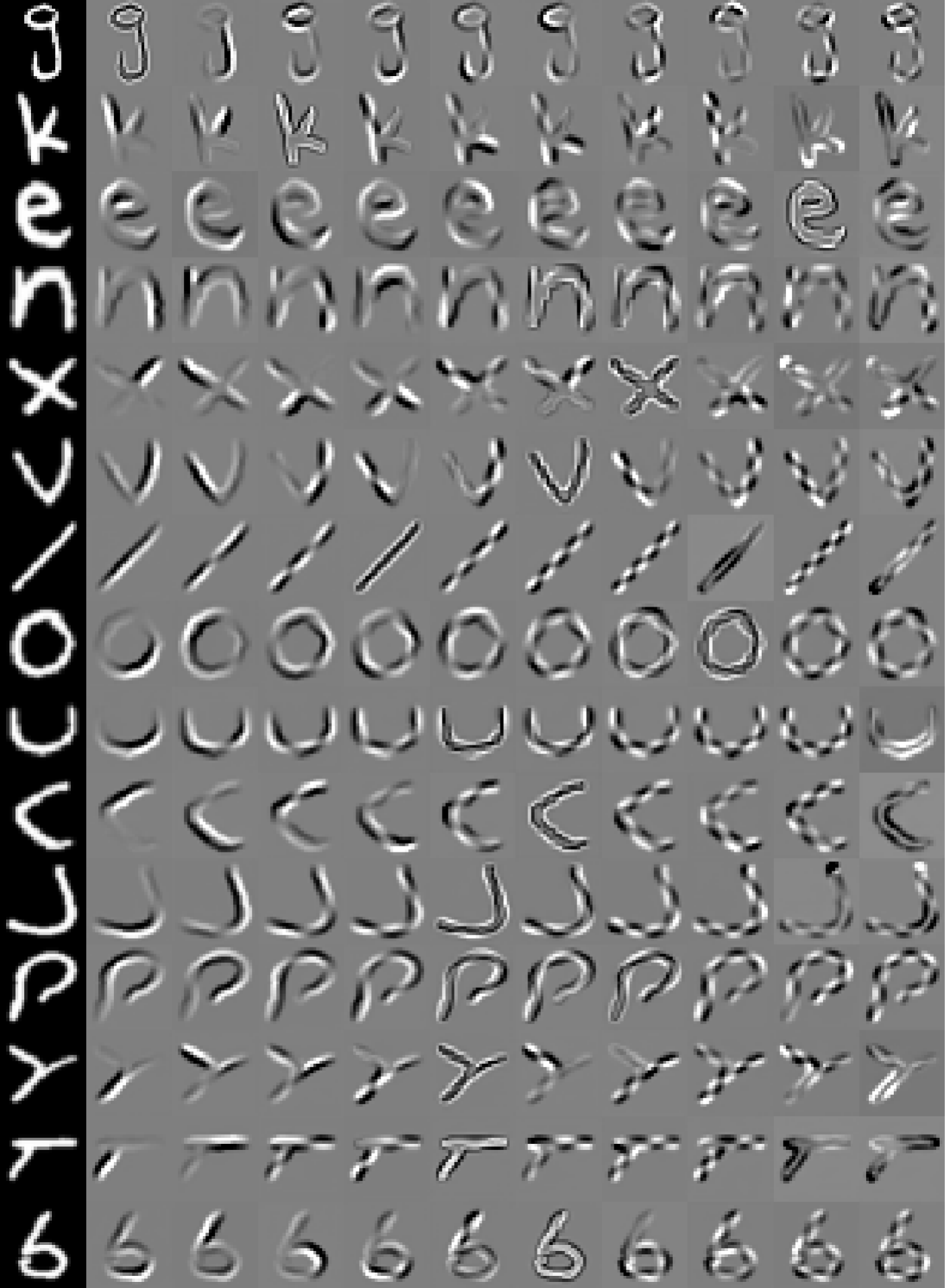}
    \caption*{\textcolor{reflink}{\textbf{(b)}}}
\end{subfigure}
\hfill
\begin{subfigure}[c]{0.32\linewidth}
    \centering
    \includegraphics[width=\linewidth]{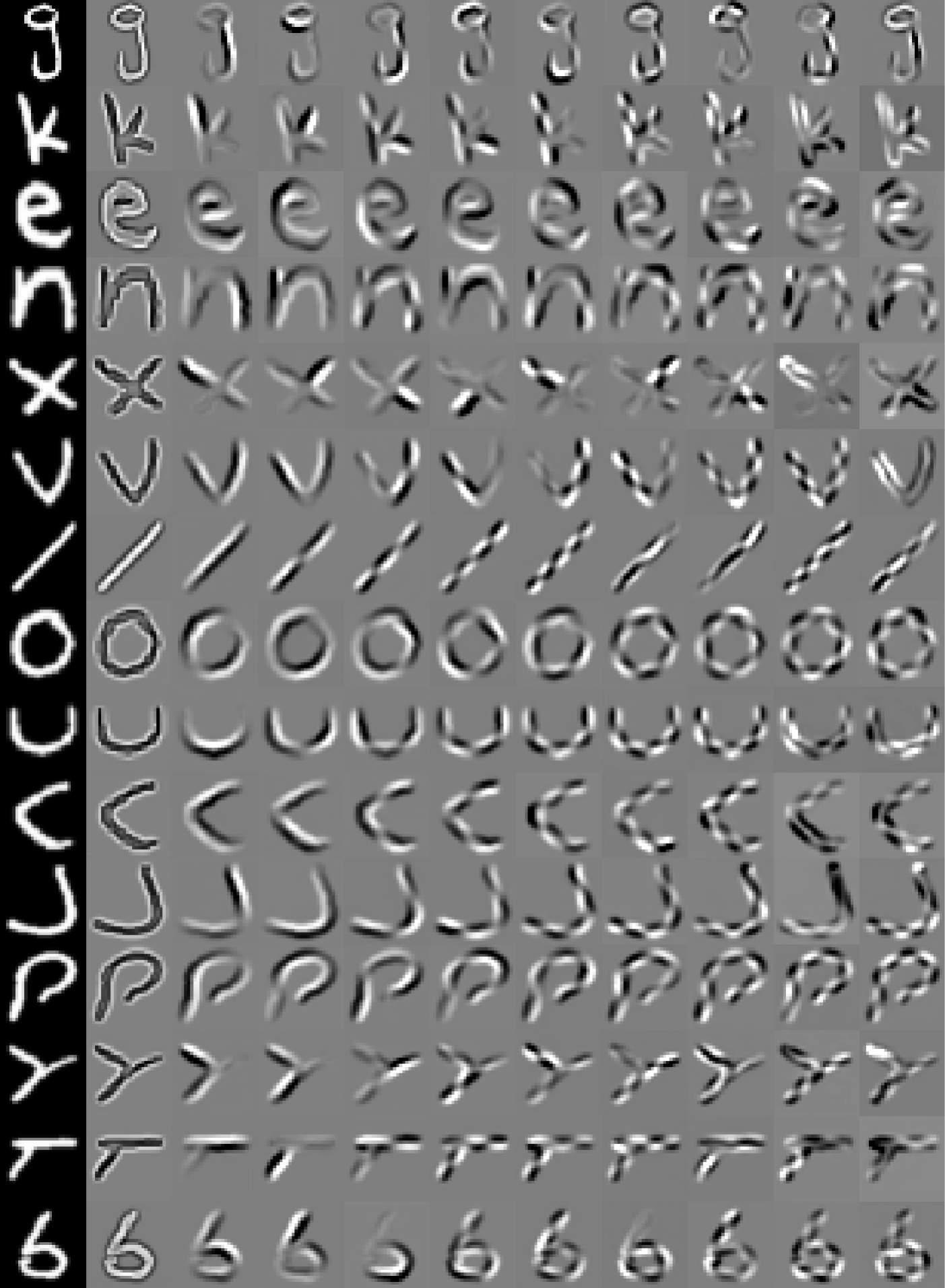}
    \caption*{\textcolor{reflink}{\textbf{(c)}}}
\end{subfigure}

\vspace{-5pt}
\caption{\textbf{IGM eigenvectors across transition scales.}
Leading CDC eigenvectors of the MNIST-trained model,
evaluated on unseen EMNIST letters at
\textcolor{reflink}{\textbf{(a)}} $\tau=0.05$,
\textcolor{reflink}{\textbf{(b)}} $\tau=0.1$, and
\textcolor{reflink}{\textbf{(c)}} $\tau=0.2$.
Larger scales capture progressively larger,
morphology-aligned variations.}
\label{fig:emnist-scale}
\end{figure}

\begin{figure}[h]
\centering

\begin{subfigure}[c]{0.47\linewidth}
    \centering
    \includegraphics[width=\linewidth]{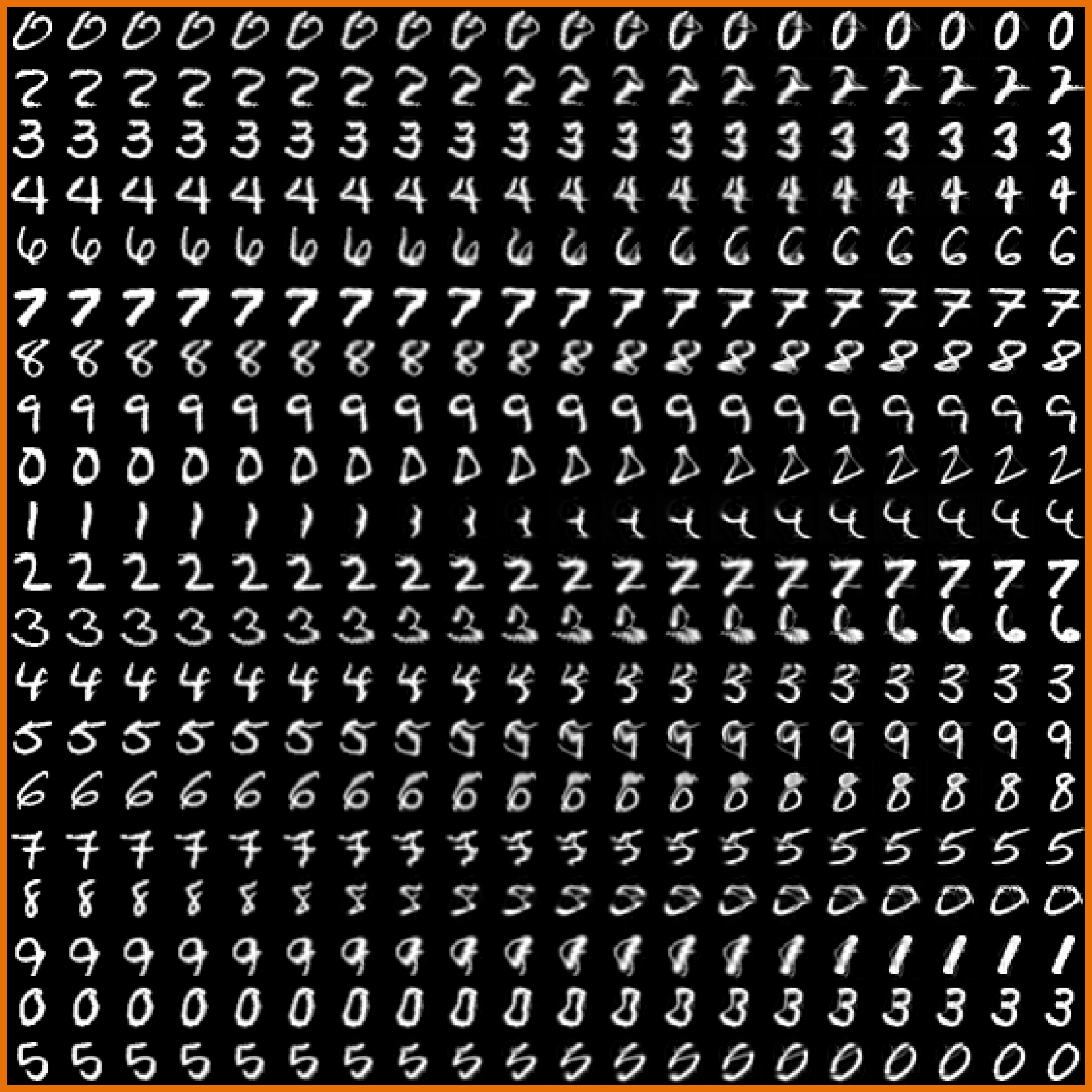}
    \caption*{\textcolor{reflink}{\textbf{(a)}}}
\end{subfigure}
\hfill
\begin{subfigure}[c]{0.47\linewidth}
    \centering
    \includegraphics[width=\linewidth]{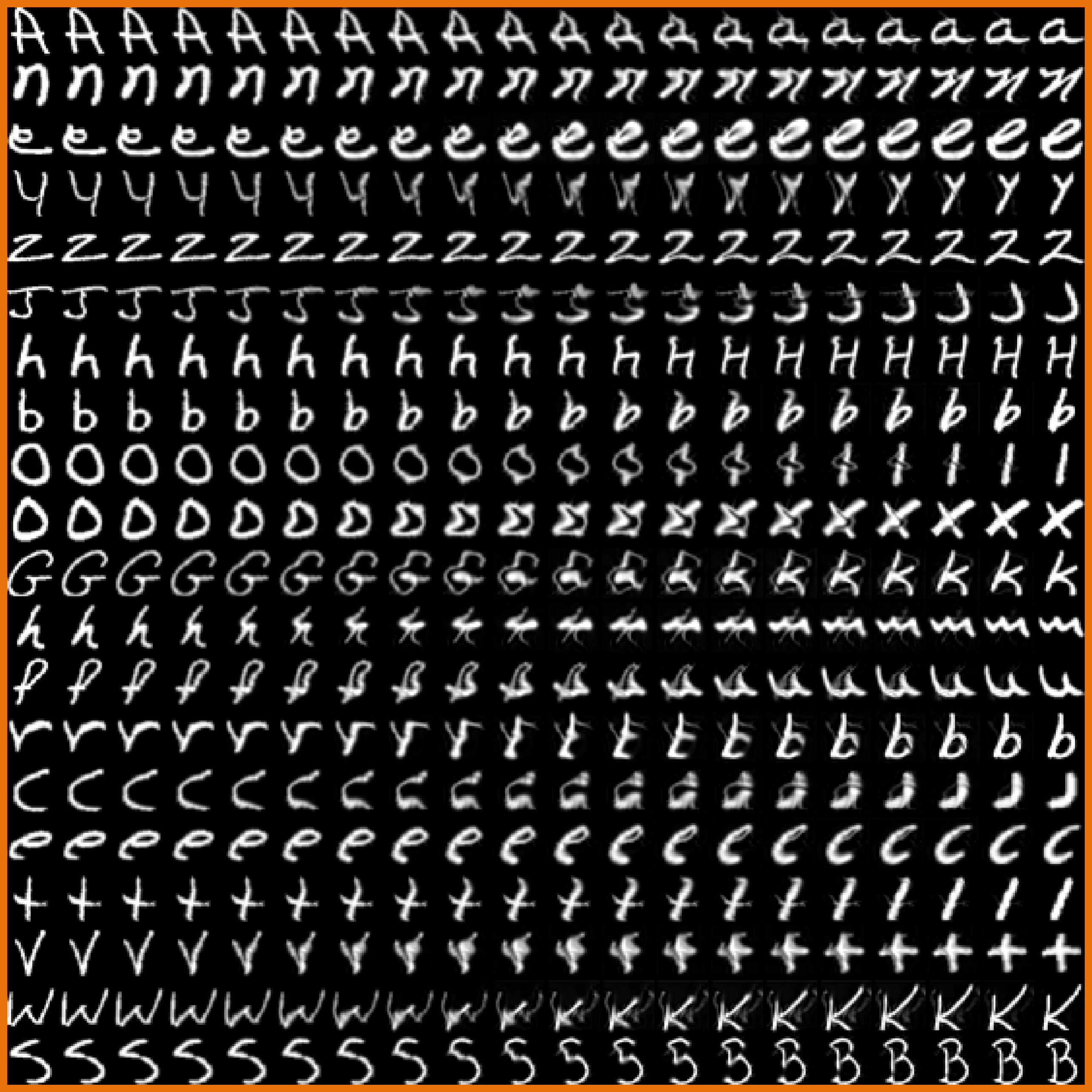}
    \caption*{\textcolor{reflink}{\textbf{(b)}}}
\end{subfigure}

\vspace{-5pt}
\caption{\textbf{Additional IGM interpolation paths.}
Paths between randomly sampled source--target pairs
from held-out images in \textcolor{reflink}{\textbf{(a)}} MNIST and \textcolor{reflink}{\textbf{(b)}} EMNIST. The MNIST-trained model produces smooth, structured transformations on both datasets, without fine-tuning on EMNIST.}
\label{fig:additional-image-interpolations}
\end{figure}

\subsection{Natural image geometry from videos}
\label{app:videos}

\paragraph{Data and transition process.}
We use the processed CelebV-HQ clips \citep{zhu2022celebvhq}, decoding RGB frames with their original presentation timestamps and resizing them to $64\times64$, giving ambient dimension $D=3\cdot64^2=12{,}288$. 
We use a $95\%$/$5\%$ split for training and validation.
Excluding clips longer than $20$ seconds leaves $33{,}594$ training clips and $1{,}825$ validation clips.

IGM is trained on forward pairs $(X,Y)=(x_t,x_{t+\tau})$ from a single clip,
with $\tau\in[0.04,0.5]$ seconds. 
We pick a source frame uniformly among those with an eligible successor, sample a lag $\tau$ log-uniformly in the range, and choose the successor closest
in log-lag.
We condition the network and normalise the loss with the actual timestamp difference in seconds, rather than the requested lag or a frame-index offset.
In each training epoch, we sample just one pair per clip, so use a negligible fraction of the available training data even when training for over 1000 epochs.

\paragraph{Gaussian perturbation baseline.}
LD-RMM uses still frames from the same training clips, with one uniformly sampled frame per selected clip. For a clean frame $X$, we generate
\begin{equation}
  \widetilde X=X+\sqrt{2b}\,\xi,
  \qquad \xi\sim\mathcal N(0,I_D),
  \qquad b\sim\operatorname{Unif}[4.5\times10^{-6},1.125],
  \label{eq:celebv-gaussian-kernel}
\end{equation}
so the corresponding noise standard deviation ranges from $0.003$ to $1.5$ in $[0,1]$ pixel units. 
LD-RMM learns the mean-centred posterior field
$\operatorname{Cov}[X\mid\widetilde X=\widetilde x,b]/(2b)$.
The forward Gaussian field would be the constant $I_D$ and would contain no image-dependent geometry.

\paragraph{Field estimation and training.}
Both estimators use the same scale-conditioned U-Net, with base width $64$,
channel multipliers $(1,2,3,4)$, two residual blocks per encoder resolution,
and three per decoder resolution. Single-head self-attention is applied at
downsampling factor $8$ and in the middle block.
To condition the residual blocks on time, we form a $64$-dimensional sinusoidal embedding of the scale (divided by $0.04$, which represents one frame at 25 frames per second) and apply a two-layer MLP of width $256$ to condition the residual blocks.
The output has $3(r+1)$ spatial channels: one RGB displacement mean
$a_\tau^\theta(q)$ and $r$ RGB factor rows, flattened to
$F_\tau^\theta(q)\in\mathbb R^{r\times D}$. We represent
\begin{equation}
  G_\tau^\theta(q)
  =F_\tau^\theta(q)^\top F_\tau^\theta(q)+\epsilon^\theta I_D.
  \label{eq:celebv-field}
\end{equation}
Here $\epsilon^\theta=\operatorname{softplus}(\rho)+10^{-8}$ is a single learned
scalar shared across images and scales, initialized at $0.01$. 
So $r$ is the rank of the structured component, while the complete covariance is full rank.
For IGM, $q=X$ and $a_\tau^\theta(X)$ predicts $Y-X$, so the endpoint mean is $m_\tau^\theta(X)=X+a_\tau^\theta(X)$.
The mean and factor rows share the U-Net and are learned jointly.

For the objective, we use the $\lambda$-regularised LogDet, which corresponds to Gaussian cross-entropy.
To specify the objective, let $d$ be the target displacement, $u$ its training
scale ($\tau$ for IGM and $b$ for LD-RMM), and
$z=(d-a_u^\theta(q))/\sqrt{2u}$.
Writing $C_\lambda=G_u^\theta(q)+\lambda I_D$, we define sample loss as cross-entropy in bits/dimension:
\begin{equation}
 \begin{aligned}
 \ell_\lambda(q,d,u)=\frac{1}{2D\log 2}\Big[
 &D\log(2\pi)+D\log(2u)+\log\det C_\lambda\\
 &+z^\top C_\lambda^{-1}z
 +\lambda\operatorname{tr}(C_\lambda^{-1})\Big].
 \end{aligned}
 \label{eq:celebv-ce}
\end{equation}
The smoothing $\lambda$ is added to both the model CDC and the rank-one
target $zz^\top$, producing the trace term in \Cref{eq:celebv-ce}.
When $\lambda=0$, this is just Gaussian NLL.

For the rank study in \cref{fig:video_results}, we train
$r\in\{16,32,64,128,256,512\}$ with $\lambda=0.01$ for $60{,}000$
optimizer steps. We use Adam with $(\beta_1,\beta_2)=(0.9,0.999)$,
$\epsilon=10^{-8}$, zero weight decay, and constant learning rate
$5\times10^{-5}$.
Each update uses $384$ sampled clips, contributing one real transition per clip for IGM or one corrupted still frame per clip for LD-RMM. 
Comparisons use the final model after $60{,}000$ steps, rather than selecting a checkpoint by its best score. 

\paragraph{Gaussian NLL and evaluation smoothing.}
The score for transition $i$ is the exact negative
log-density of $d_i$ under the Gaussian surrogate
\begin{equation}
 d_i\mid q_i,\tau_i\ \sim
 \mathcal N\!\left(a_{\tau_i}^\theta(q_i),\,
                    2\tau_iG_{\tau_i}^\theta(q_i)\right),
 \qquad
 \mathrm{NLL}=\frac{1}{N}\sum_{i=1}^N
                  \ell_0(q_i,d_i,\tau_i).
 \label{eq:celebv-nll}
\end{equation}
This is reported in bits per ambient dimension.
It measures mean prediction, covariance calibration, and directional structure together.
GGM is read at a fixed bandwidth of $\tau_i = 1.6 \times 10^{-3}$, chosen so the noise level matches the observed per-pixel displacement scale (about $0.05$ per pixel).
This choice is not very significant, since NLL and EV barely change over the trained bandwidth range (CE by $0.053$ bits/dim, EV by $0.014$).

\paragraph{Explained transition variance.}
Let $v_{ij}$ be the $j$th unit eigenvector of
$F_{\tau_i}^\theta(q_i)^\top F_{\tau_i}^\theta(q_i)$, ordered by decreasing
eigenvalue, and $P_{ik}=\sum_{j=1}^k v_{ij}v_{ij}^\top$.
For $k\leq r$, we report the explained variance as a fraction
\begin{equation}
 \mathrm{EV}(k)
 =\frac{\sum_{i=1}^N
                 \|P_{ik}(y_i-x_i)\|_2^2 (2\tau_i)^{-1}}
        {\sum_{i=1}^N
                 \|y_i-x_i\|_2^2(2\tau_i)^{-1}}.
 \label{eq:celebv-explained-variance}
\end{equation}
We aggregate the normalised energies over all transitions before dividing, rather than averaging per-pair fractions.
The statistic measures how much observed motion lies in the leading learned directions.
The profile is invariant to positive rescaling of the CDC at a fixed
readout and is unaffected by adding an isotropic ridge, which changes
eigenvalues but not eigenvectors. 
We report $k\in\{8,16,32,64,128,256,512\}$ up to the model rank.
Small $k$ probes the ordering of useful directions, whereas $k=r$ probes the coverage of the learned span. 

\paragraph{Isotropic reference scores.}
We add two isotropic references to the NLL figures, showing the best possible scores attainable by isotropic models, both homogeneous and inhomogeneous.
For an isotropic raw-displacement covariance $\sigma_i^2I_D$, let
$w_i=\|y_i-x_i\|_2^2/D$. 
The NLL in bits per dimension is
$[\log(2\pi)+\log\sigma_i^2+w_i/\sigma_i^2]/(2\log2)$.
For the \textbf{homogeneous model}, the optimal constant variance is $\bar w=N_b^{-1}\sum_i w_i$.
The \textit{inhomogeneous model} uses a per-transition scale $\sigma_i^2=w_i$: it knows each held-out transition's magnitude and is an \textit{oracle reference}, not an independently fitted predictor, but records the best score possibly attainable by an isotropic model.
We therefore have
\begin{equation}
 \mathrm{NLL}_{\mathrm{hom}}
 =\tfrac12\log_2(2\pi\mathrm e\,\bar w),
 \qquad
 \mathrm{NLL}_{\mathrm{inhom}}
 =\frac{1}{2N_b}\sum_{i=1}^{N_b}\log_2(2\pi\mathrm e\,w_i).
 \label{eq:celebv-isotropic}
\end{equation}

\paragraph{Training-smoothing ablation.}
We fix $r=128$ and vary $\lambda_t\in\{0,0.01,0.05,0.1\}$.
As shown in \Cref{fig:celebv-lambda-t-ablation}, increasing IGM smoothing from $0$ to $0.1$ raises $\mathrm{EV}(k)$ for small $k$, concentrating more observed motion in the leading directions, while minimally affecting NLL.
We visualise the CDC eigenvectors in for $\lambda=0$ (\cref{fig:celebv_tangents_0.0}) and $\lambda=0.1$ (\cref{fig:celebv_tangents_0.1}), which show the effect of $\lambda$ in reordering the leading CDC directions.

\begin{figure}[h]
  \centering

    \begin{minipage}[t]{0.3\linewidth}
        \centering
        \includegraphics[width=\linewidth]
            {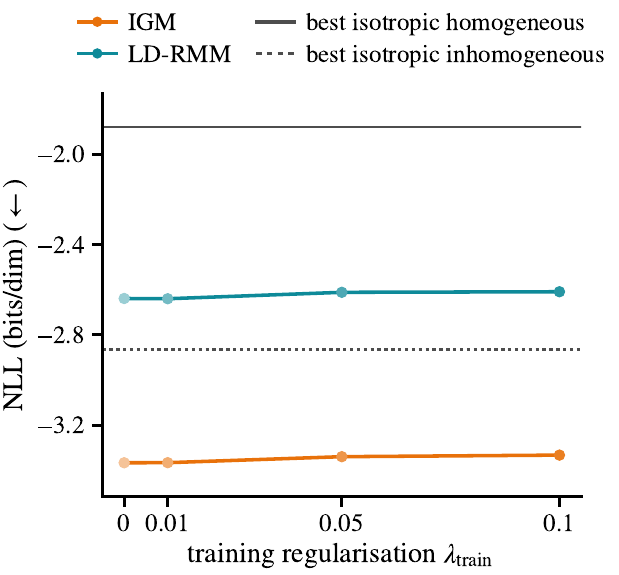}\par
        \vspace{0.15em}
        \makebox[\linewidth][l]{%
            \hspace*{0.58\linewidth}%
            \makebox[0pt][c]{\small\textcolor{reflink}{\textbf{(a)}}}%
        }
    \end{minipage}
    \begin{minipage}[t]{0.3\linewidth}
        \centering
        \includegraphics[width=\linewidth]
            {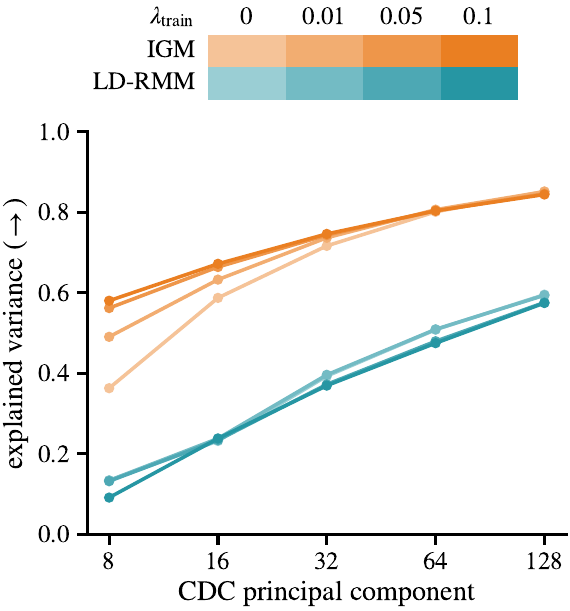}\par
        \vspace{0.15em}
        \makebox[\linewidth][l]{%
            \hspace*{0.58\linewidth}%
        \makebox[0pt][c]{\small\textcolor{reflink}{\textbf{(b)}}}%
        }
    \end{minipage}
    \vspace{-0.5em}
    
  \caption{\textbf{Training-smoothing ablation on CelebV-HQ.}
  Left: validation NLL, varying training smoothing $\lambda$ at rank $128$.
  Right: explained transition variance in the leading $k$ CDC directions.
  \textcolor{cdcorange}{\textbf{Orange}}: IGM;
\textcolor{lerpteal}{\textbf{teal}}: LD-RMM.
  Isotropic NLL references are defined in
  \cref{eq:celebv-isotropic}.}
  \label{fig:celebv-lambda-t-ablation}
\end{figure}

\clearpage
\begin{figure}[p]
    \centering
    \includegraphics[width=1\linewidth]{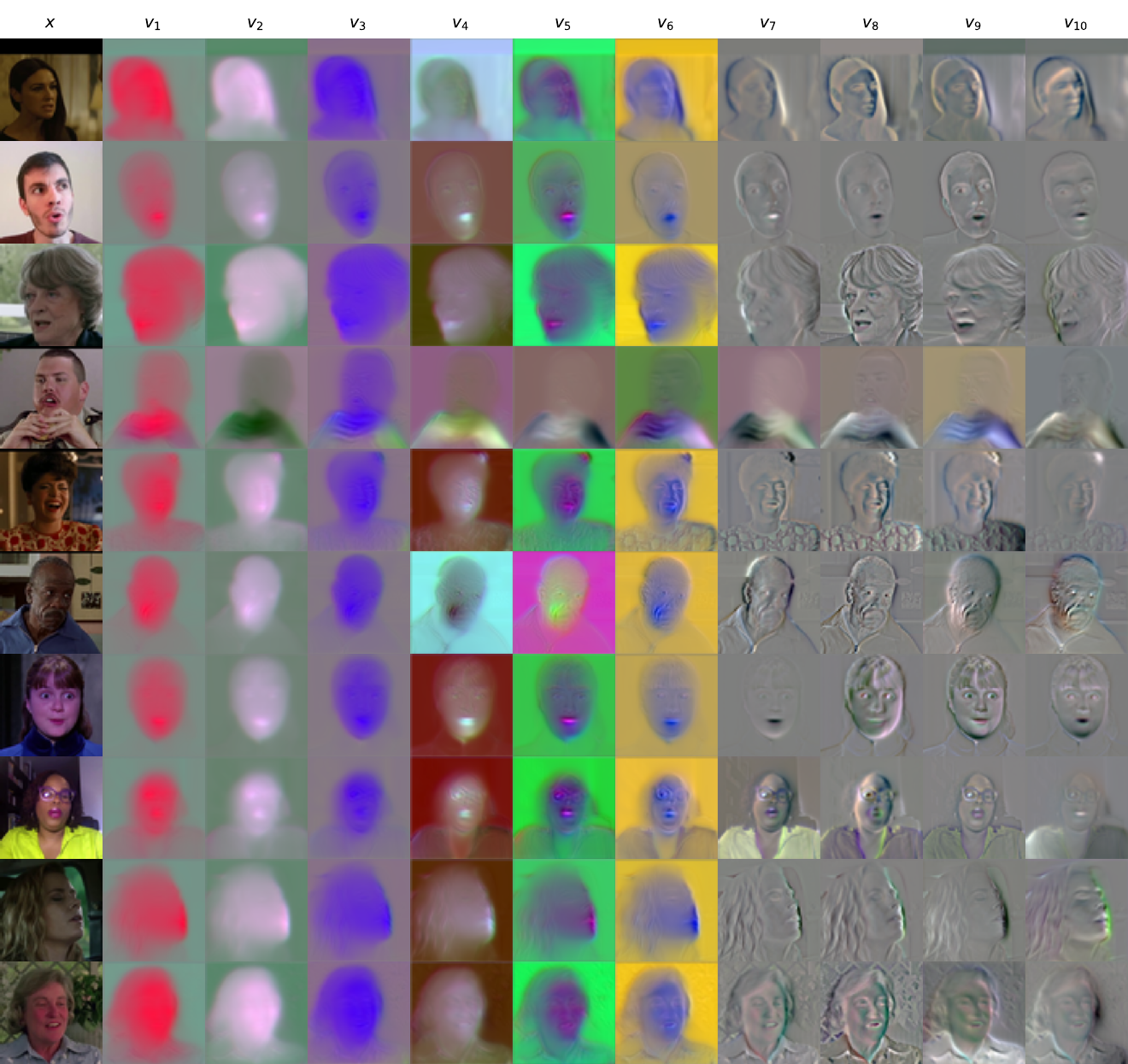}
    \caption{\textbf{Intrinsic IGM geometry of images, with unregularised LogDet.}
    Leading CDC eigenvectors on unseen validation images, with no training smoothing.
    }
    \label{fig:celebv_tangents_0.0}
\end{figure}
\clearpage
\clearpage
\begin{figure}[p]
    \centering
    \includegraphics[width=\linewidth]{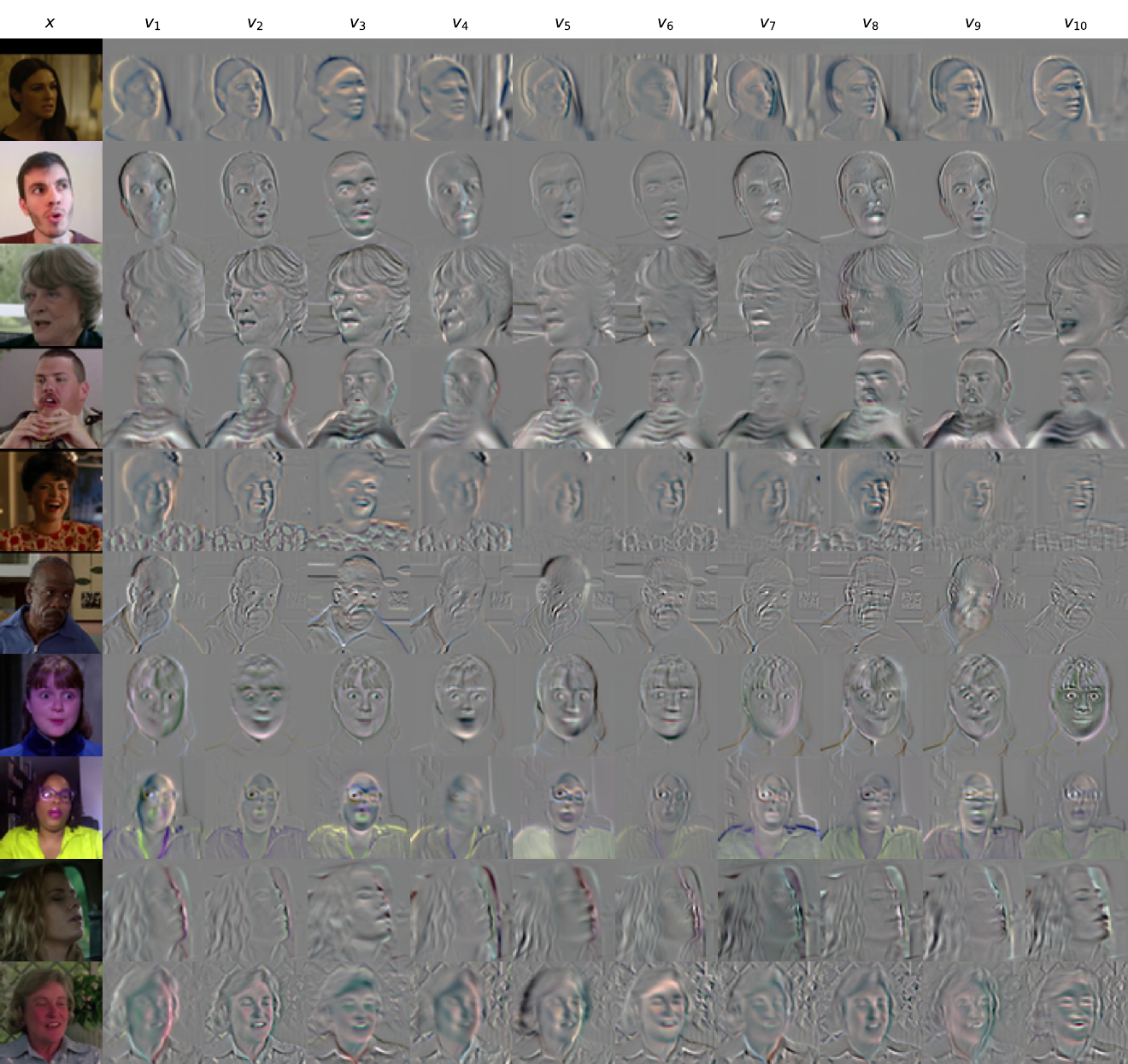}
    \caption{\textbf{Intrinsic IGM geometry of images, with LogDet regularisation $\lambda = 0.1$.}
    Leading CDC eigenvectors on unseen validation images, with training smoothing $\lambda = 0.1$.
    }
    \label{fig:celebv_tangents_0.1}
\end{figure}
\clearpage
\subsection{Equipping latent representations with process geometry}
\label{app:md}

\paragraph{Data and representation.}
We use $200$ dynamic importance sampling trajectories of adenylate kinase (AdK) \citep{perilla2011dims,beckstein2009adk}, comprising $19{,}691$ frames sampled at $1\,\mathrm{ps}$ intervals and split by trajectory into $160/20/20$ training, validation, and test trajectories.
AdK consists of a relatively stable CORE and the mobile \textcolor{lidgreen}{\textbf{LID}} and \textcolor{nmpblue}{\textbf{NMP}} domains, whose relative motion characterises the open--closed transition \citep{beckstein2009adk}.
Training uses all valid within-trajectory pairs $(X_t,X_{t+\tau})$ at integer lags $\tau\in\{1,\ldots,32\}\,\mathrm{ps}$, giving $419{,}136$ transitions.
We freeze the pretrained SLAE encoder $\Phi$ and decoder $\Psi$ \citep{chen2025slae}.
SLAE represents each residue by an embedding of its local all-atom environment and decodes the resulting embeddings back to molecular coordinates.
The encoder gives $Y_t=\Phi(X_t)\in\R^{214\times128}$, or $D=27{,}392$ latent coordinates, which we standardise using training-set statistics before fitting the transition models.

\paragraph{Architecture and training.}
The mean and CDC networks preserve the $214\times128$
residue structure through residue-shared projections,
learned positional embeddings, and residual blocks
combining LayerNorm, four-head self-attention, and
FiLM-conditioned MLPs with a sinusoidal embedding of $\tau$.
The mean network uses width $128$ and three residual
blocks; the CDC network uses width $256$ and four.
We first fit the conditional mean $m_\tau^\phi(Y_t)\approx\E[Y_{t+\tau}\mid Y_t]$ with Adam at learning rate $2\times10^{-4}$, batch size $64$, $200$ batches per epoch, zero weight decay, and gradient clipping at norm $1$, for at most $400$ epochs with patience $20$ on validation displacement.
We then freeze $m^\phi$ and train CDCs of rank $r\in\{16,32,64,128,256,512\}$ on the centred residuals $Y_{t+\tau}-m_\tau^\phi(Y_t)$.
This two-stage procedure has the same population optimum as joint Gaussian regression, while allowing the mean to be fitted once and shared across the rank sweep.
The CDCs use IGM with  low-rank parametrisation $G_\tau^\theta=M_\tau^\theta(M_\tau^\theta)^\top+\epsilon^\theta I$.  Adam at learning rate $10^{-4}$, and otherwise the same batch construction, for at most $300$ epochs with patience $15$.
Checkpoints are selected by Gaussian NLL on $4096$ validation transitions, and $\epsilon^\theta$ is initialized at $10^{-2}$.

\paragraph{Held-out transition geometry.}
For \Cref{fig:md}\textcolor{reflink}{\textbf{(a)}},
we evaluate all ranks on the same $4096$ transitions
from the $20$ held-out test trajectories.
We report Gaussian NLL as in \eqref{eq:celebv-nll},
scoring the successor state under
$\mathcal N(m_\tau^\phi(Y_t),2\tau G_\tau^\theta(Y_t))$
in bits per standardised latent dimension.
We compute $95\%$ confidence intervals by resampling
the $20$ test trajectories $10{,}000$ times, keeping
transitions from each trajectory together.

Unlike the uncentred video score in
\eqref{eq:celebv-explained-variance}, we measure
residual variability after subtracting the predicted
conditional mean.
For transition $i$, let
\[
\delta_i
=
\frac{Y_{t_i+\tau_i}-m_{\tau_i}^\phi(Y_{t_i})}
{\sqrt{2\tau_i}},
\]
and let $P_{ik}$ project onto the leading $k$
eigenvectors of $G_{\tau_i}^\theta(Y_{t_i})$.
For $k\leq r$, we report
\begin{equation}
\label{eq:md-explained-residual-variance}
\mathrm{EV}_{\mathrm{res}}(k)
=
\frac{\sum_i\|P_{ik}\delta_i\|_2^2}
{\sum_i\|\delta_i\|_2^2}.
\end{equation}
This measures the fraction of residual variability
captured by the leading CDC directions.

\paragraph{Endpoint-conditioned bridge paths.}
For \Cref{fig:md}\textcolor{reflink}{\textbf{(b,c)}}, we condition the learned transition model on the observed endpoints of each held-out trajectory.
Starting from the initial latent state $f_0=Y_0$, we first roll out the learned mean on a $16\,\mathrm{ps}$ grid,
\[
f_{k+1}=m_{\tau_k}^\phi(f_k),
\]
where $\tau_k=16\,\mathrm{ps}$ except for a shorter final step when required.
Along this free-mean path we freeze the predictive covariances
$Q_k=2\tau_kG_{\tau_k}^\theta(f_k)$.
These define a time-inhomogeneous additive Gaussian surrogate, whose exact conditional mean given the observed terminal state $Y_T$ is
\begin{equation}
\bar Y_j
=
f_j+
\left(\sum_{k<j}Q_k\right)
\left(\sum_{k<K}Q_k\right)^{-1}
(Y_T-f_K).
\label{eq:md-bridge}
\end{equation}
We solve the corresponding linear system in double precision and decode the resulting bridge with the frozen SLAE decoder.
For evaluation, predicted and observed paths are resampled to $80$ points by latent arclength and compared by discrete Fr\'echet distance using C$\alpha$ Kabsch-RMSD as the ground distance \citep{seyler2015psa}.
\Cref{eq:md-bridge} is the exact bridge mean of the frozen Gaussian surrogate, rather than of the nonlinear learned transition process.
For the exploratory path analysis, we focus on CDC ranks 32–128 and evaluate eight lags on $20$ validation trajectories. Rank $32$ at $16$ ps has the lowest observed mean Fréchet distance in this subset.

\paragraph{Transport to molecular coordinates.}
For \Cref{fig:md}\textcolor{reflink}{\textbf{(c)}}, bottom, we transport the learned CDC back to molecular coordinates through the frozen SLAE decoder $\Psi$.
Let $\Psi$ denote the map from standardised latent coordinates to CORE-aligned C$\alpha$ coordinates, including inverse latent standardisation and rigid alignment.
Following the observation-equivariance rule of \eqref{eq:conjugation_pushfoward_cdc}, the linearised predictive covariance at latent state $Y$ is
\begin{equation}
\Sigma_X(Y)
=
\bigl(D_{m_\tau^\phi(Y)}\Psi\bigr)
\bigl[2\tau G_\tau^\theta(Y)\bigr]
\bigl(D_{m_\tau^\phi(Y)}\Psi\bigr)^\top
\in\mathbb S_+^{3\cdot214}.
\label{eq:md-decoder-pushforward}
\end{equation}
For the main-figure ellipsoids, we use the selected
rank-$32$ CDC at $\tau=16\,\mathrm{ps}$ and visualise
each residue's $3\times3$ diagonal block as a
one-standard-deviation mobility ellipsoid.

\begin{figure}[h]
    \centering

    \makebox[\linewidth][c]{%
        \includegraphics[width=0.19\linewidth]{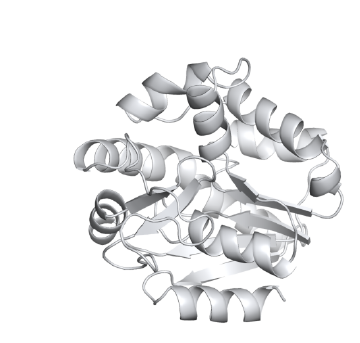}%
        \hfill
        \includegraphics[width=0.19\linewidth]{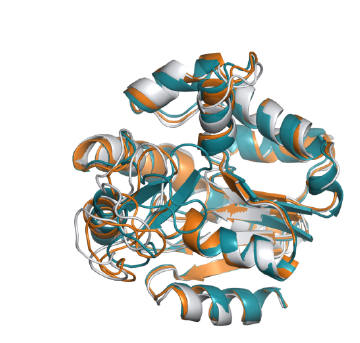}%
        \hfill
        \includegraphics[width=0.19\linewidth]{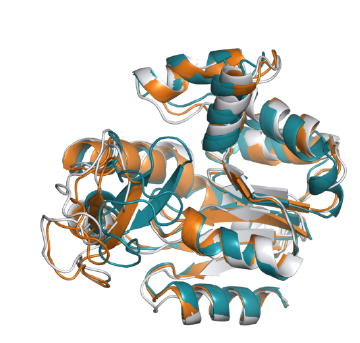}%
        \hfill
        \includegraphics[width=0.19\linewidth]{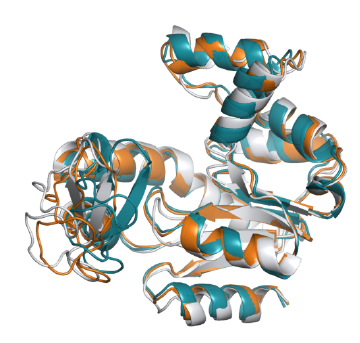}%
        \hfill
        \includegraphics[width=0.19\linewidth]{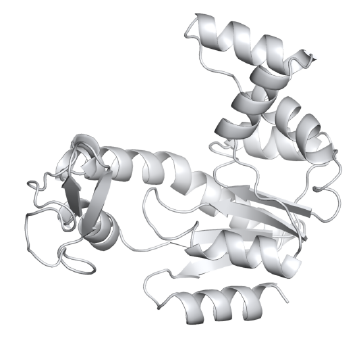}%
    }

    \vspace{0.35em}

    \makebox[\linewidth][c]{%
        \includegraphics[width=0.19\linewidth]{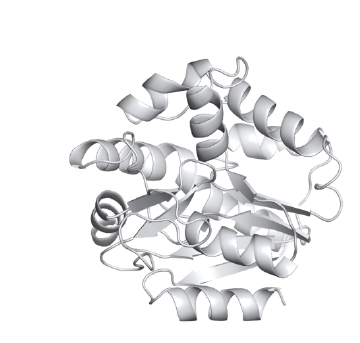}%
        \hfill
        \includegraphics[width=0.19\linewidth]{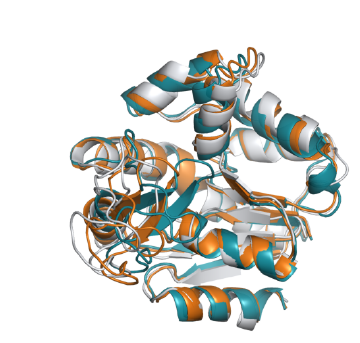}%
        \hfill
        \includegraphics[width=0.19\linewidth]{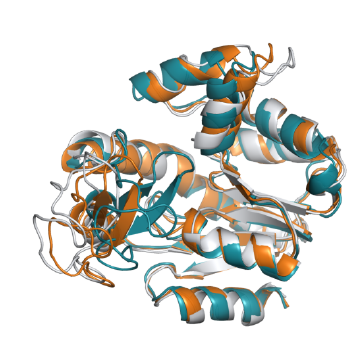}%
        \hfill
        \includegraphics[width=0.19\linewidth]{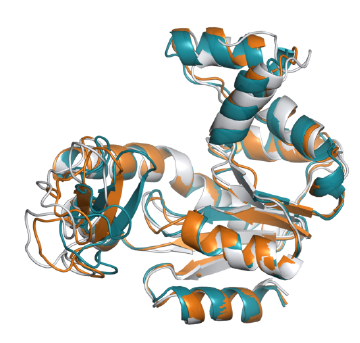}%
        \hfill
        \includegraphics[width=0.19\linewidth]{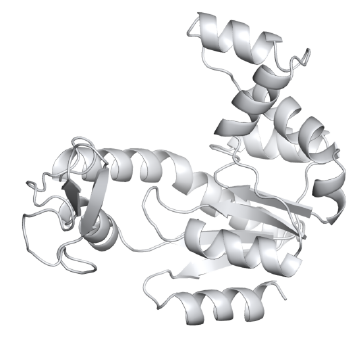}%
    }

    \vspace{-0.5em}
    \caption{\textbf{Decoded endpoint-conditioned paths.}
    Two held-out MD transitions (trajectories 111 and 145; top and bottom) compared with the decoded IGM bridge and latent linear interpolation (LERP) at matched path progress. Bridges use rank    $r=32$ and lag $\tau=16\,\mathrm{ps}$.
    \textcolor{cdcorange}{\textbf{Orange}}: IGM;
    \textcolor{lerpteal}{\textbf{teal}}: LERP;
    \textcolor{mdheld}{\textbf{grey}}: held-out MD.}
    \label{fig:md-app-decoded-paths}
\end{figure}

\begin{figure}[h]
    \centering
    \makebox[\linewidth][c]{%
        \includegraphics[width=0.19\linewidth]{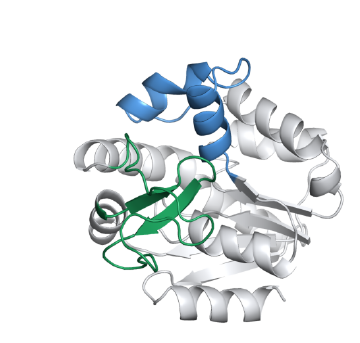}%
        \hfill
        \includegraphics[width=0.19\linewidth]{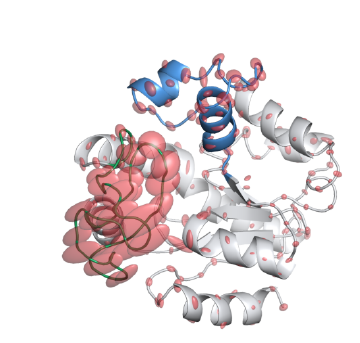}%
        \hfill
        \includegraphics[width=0.19\linewidth]{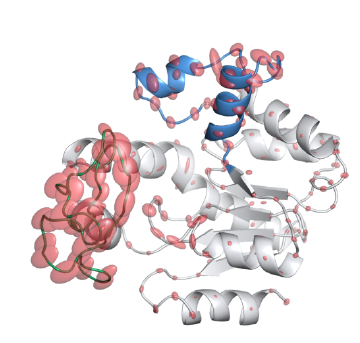}%
        \hfill
        \includegraphics[width=0.19\linewidth]{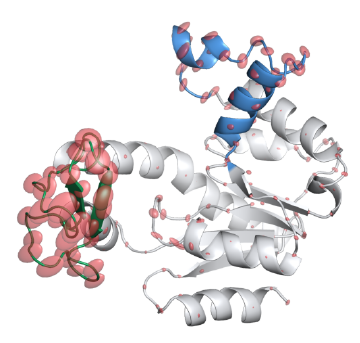}%
        \hfill
        \includegraphics[width=0.19\linewidth]{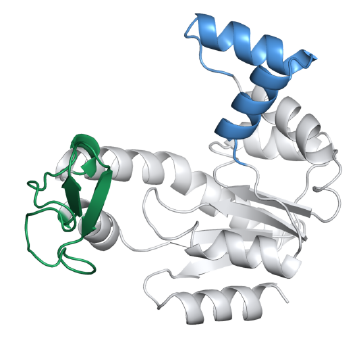}%
    }
    \vspace{0.35em}
    \makebox[\linewidth][c]{%
        \includegraphics[width=0.19\linewidth]{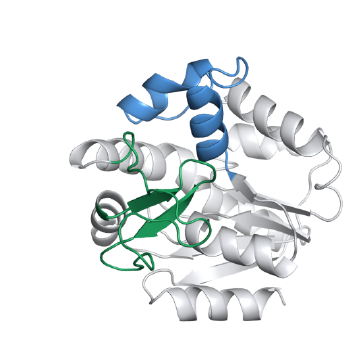}%
        \hfill
        \includegraphics[width=0.19\linewidth]{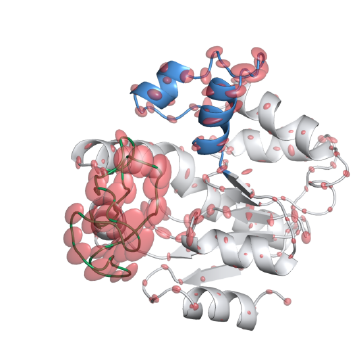}%
        \hfill
        \includegraphics[width=0.19\linewidth]{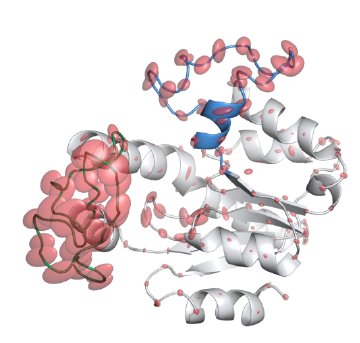}%
        \hfill
        \includegraphics[width=0.19\linewidth]{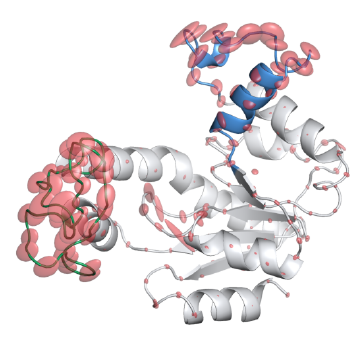}%
        \hfill
        \includegraphics[width=0.19\linewidth]{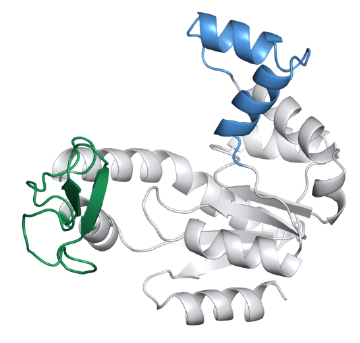}%
    }

    \vspace{-0.5em}
   \caption{\textbf{Learned process geometry in molecular coordinates.}
    Per-residue mobility along held-out trajectories
    $111$ (top) and $145$ (bottom), shown as
    one-standard-deviation ellipsoids.
    The \textcolor{lidgreen}{\textbf{LID}} and
    \textcolor{nmpblue}{\textbf{NMP}} domains exhibit greater
    predicted mobility than the
    \textcolor{corescaf}{\textbf{CORE}}.}
    \label{fig:md-app-ellipsoids}
\end{figure}

\clearpage
\begin{figure}
    \vspace{-1em}
    \centering
    \includegraphics[width=\linewidth]{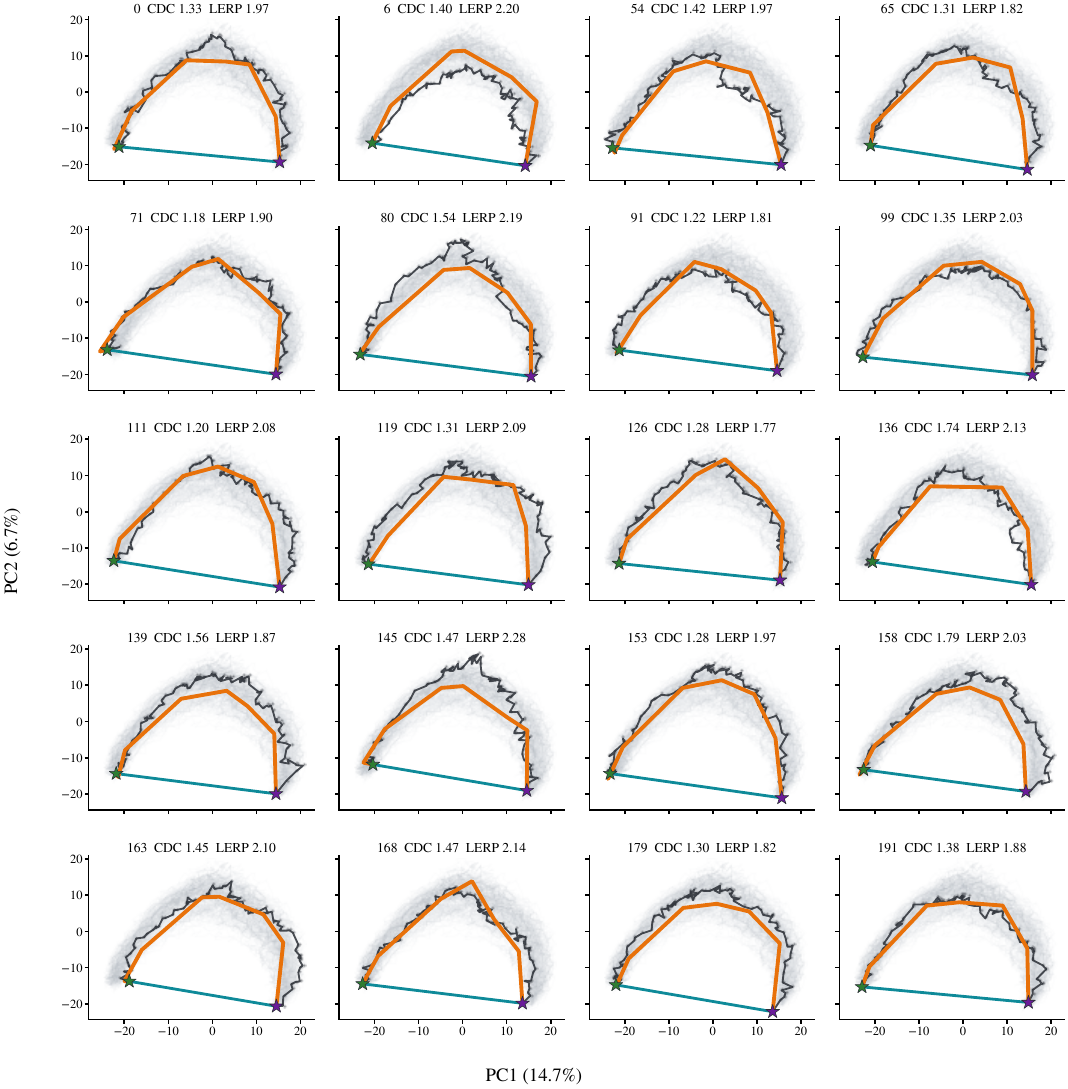}
    \caption{\textbf{Endpoint-conditioned paths on held-out MD trajectories.}
    Each panel compares a held-out trajectory
    (\textcolor{mdheld}{\textbf{grey}}), its IGM bridge
    (\textcolor{cdcorange}{\textbf{orange}}), and latent
    linear interpolation (LERP;
    \textcolor{lerpteal}{\textbf{teal}}) between the
    same closed (\textcolor{epclosed}{$\bigstar$}) and
    open (\textcolor{epopen}{$\bigstar$}) endpoints.
    Paths are projected onto the first two principal
    components of the training states.
    Headings report trajectory IDs and decoded discrete
    Fr\'echet distances in \AA, using C$\alpha$
    Kabsch-RMSD as the ground distance.
    IGM improves on LERP for all $20$ test trajectories,
    with mean distances of $1.418$ and $2.039$\,\AA,
    respectively.
    Rank $r=32$ and lag $\tau=16\,\mathrm{ps}$ were
    selected on the validation trajectories.}
    \label{fig:app-pca-panel}
    \vspace{-1em}
\end{figure}
\clearpage

\end{document}